\documentclass{article}

\usepackage{authblk}
\usepackage{breakcites}
\usepackage{booktabs}
\usepackage{amsthm,twoopt,mathtools}
\usepackage{times,ifthen}
\usepackage{fancyhdr}
\usepackage{xr,placeins}

\title{Asymptotic convergence of iterative optimization algorithms}
\date{}

\author[$\star$]{Randal Douc}
\author[$\dag$]{Sylvain Le Corff}

\affil[$\star$]{{\small Samovar, T\'el\'ecom SudParis, Institut Polytechnique de Paris}}
\affil[$\dag$]{{\small LPSM, Sorbonne Universit\'e, UMR CNRS 8001, 4 Place Jussieu, 75005 Paris}}

\usepackage{geometry}
\usepackage[utf8]{inputenc}
\usepackage[english]{babel}
\usepackage{amsmath}
\usepackage{amsfonts}
\usepackage{amssymb}
\usepackage{caption}
\usepackage{array}
\usepackage{framed}
\usepackage{multirow}
\usepackage{bbm}
\usepackage{xcolor}
\usepackage{mathrsfs}
\usepackage{pdfsync,xargs}
\usepackage[shortlabels]{enumitem}
\usepackage[textwidth=80pt]{todonotes}

\usepackage{aliascnt}
\usepackage{hyperref}
\usepackage{cleveref}

\usepackage{mleftright}
\usepackage{stmaryrd}
\usepackage{dsfont}

\usepackage{mismath}
\usepackage{upgreek}
\usepackage{yhmath}
\usepackage[T3,T1]{fontenc}

\usepackage[ruled]{algorithm2e}
\usepackage{appendix}

\newenvironment{hyp}[2]{
\begin{enumerate}[label=\textbf{\sf(#1\arabic*)},resume=hyp#2]\begin{sf}}
{\end{sf}\end{enumerate}}
\crefname{hyp}{}{ass}
\Crefname{hyp}{}{Ass}

\newtheorem{theorem}{Theorem}
\newtheorem{corollary}{Corollary}
\newtheorem{lemma}{Lemma}[section]
\newtheorem{proposition}{Proposition}

\theoremstyle{remark}

\newtheorem{remark}{Remark}

\newtheorem{example}{Example}

\newenvironment{manualexample}[1]{%
  \renewcommand\theexample{#1}%
  \example
}{\endexample}

\newcommand{\0}{\mathbf{0}}
\newcommandx{\admiss}[1][1=f]{\mathsf{A}_{#1}}
\newcommand{\Aff}[1]{\mathrm{Aff}(#1)}

\newcommand{\argmin}{\mathrm{argmin}}

\newcommand{\Ball}[2]{\mathbf{B}\left(#1, #2 \right)}

\newcommand{\At}{\mathcal{A}_{\theta\theta'}}
\newcommand{\Bt}{\mathcal{B}_{\theta\theta'}}
\newcommand{\Av}{\mathcal{A}_\star}
\newcommand{\Bv}{\mathcal{B}_\star}

\newcommand{\An}{\mathcal{A}_n}
\newcommand{\Bn}{\mathcal{B}_n}

\newcommand{\tAt}{\tilde{\mathcal{A}}_{\theta\theta'}}
\newcommand{\tBt}{\tilde{\mathcal{B}}_{\theta\theta'}}
\newcommand{\tAv}{\tilde{\mathcal{A}}_\star}
\newcommand{\tBv}{\tilde{\mathcal{B}}_\star}

\newcommand{\tDv}{\tilde{\mathcal{D}}_\star}
\newcommand{\tAn}{\tilde{\mathcal{A}}_n}
\newcommand{\tBn}{\tilde{\mathcal{B}}_n}

\newcommand{\tS}{\tilde{\mathcal{S}}}
\newcommand{\tSn}{\tilde{\mathcal{S}}_n}
\newcommand{\tSv}{\tilde{\mathcal{S}}_\star}

\newcommandx{\binfty}[1][1=\alpha]{|b|_{\infty, #1}}
\newcommandx{\balpha}[1][1=\alpha]{|b|_{#1}}
\newcommandx{\bmuf}[2][1=\mu, 2=\alpha]{ {b_{#1, #2}}}
\newcommandx{\bmufk}[2][1=\mu, 2=\alpha]{\hat{b}_{#1, #2, M}}

\newcommandx{\couple}[2][1=\PQ, 2=\PP]{(#1 || #2)}

\renewcommand{\Conv}[1]{\mathrm{Conv}(#1)}
\newcommand{\Cset}{\mathsf C}

\newcommandx{\cteLipSto}[2][1=N, 2=M]{|b|_{\hmu_{#1}, #2, \alpha}}
\newcommandx{\cteLipStoInf}[1][1=\alpha]{B_{#1}}
\newcommandx{\ctemono}[1][1=\alpha]{L_{#1, 1}}
\newcommandx{\cteinf}[1][1=\alpha]{L_{#1, 2}}
\newcommandx{\ctesup}[1][1=\alpha]{L_{#1, 3}}

\newcommand{\donce}[1]{\partial_2 #1}
\newcommand{\dtwice}[3]{\ifthenelse{\equal{#1}{#2}}
{\partial_{2 2} #3}
{\partial_{1 2} #3}
}

\newcommandx{\diverg}[1][1=\alpha]{D_{#1}}
\newcommandx{\Domain}[1][1=\alpha]{\mathrm{Dom}_{#1}}
\newcommand{\Dset}{\mathsf{D}}

\renewcommand{\epsilon}{\varepsilon}

\renewcommand{\eqdef}{:=}

\newcommand{\eqsp}{\;}
\newcommand{\Eset}{\mathsf{E}}

\newcommandx{\falpha}[1][1=\alpha]{f_{#1}}
\newcommandx{\aei}[1][1=\alpha]{$(#1, \Gamma)$-}
\newcommandx{\GammaAlpha}[1][1=\alpha]{ \Gamma}
\renewcommand{\geq}{\geqslant}
\newcommandx{\gmuf}[1][1=\mu]{ {g_{#1}}}
\newcommand{\hmu}{\hat{\mu}}
\newcommandx{\hbinfty}[1][1=\alpha]{|\hat{b}|_{\infty, #1}}

\newcommand{\indic}[1]{\mathds{1}_{#1}}

\newcommand{\integerval}[2][1]{\llbracket #1 : #2 \rrbracket}

\newcommandx{\iteration}[1][1=\alpha]{\mathcal{I}_{#1}}
\newcommandx{\iterationK}[1][1=\alpha]{\hat{\mathcal{I}}_{#1, M}}

\newcommandx{\lbd}[2][1=]{
    \ifthenelse{\equal{#1}{}}
    {{\boldsymbol{\lambda}_{#2}}}
    {\lambda_{#1,#2}}
    }
\newcommand{\Kset}{\mathsf K}

\renewcommand{\leq}{\leqslant}

\newcommand{\lrav}[1]{\left|#1 \right|}
\newcommand{\lr}[1]{\left(#1 \right)}
\newcommand{\lrb}[1]{\left[#1 \right]}
\newcommand{\lrc}[1]{\left\{#1 \right\}}

\newcommand{\M}[2][]{\mathcal{M}^{#1}\lr{#2}}
\newcommand{\Mtxt}[2][]{\mathcal{M}^{#1}(#2)}
\newcommand{\tM}[2][]{\tilde{\mathcal{M}}_{#1}\lr{#2}}

\newcommand{\mca}{\mathcal{A}}
\newcommand{\mcb}{\mathcal{B}}
\newcommand{\tmca}{\tilde{\mathcal{A}}}
\newcommand{\tmcb}{\tilde{\mathcal{B}}}

\newcommand{\mcm}{\mathcal{M}}

\newcommand{\mcq}{\mathcal{Q}}

\newcommand{\mcr}{\mathcal{R}}

\newcommand{\mcx}{\mathcal{X}}
\newcommand{\mcy}{\mathcal{Y}}

\newcommand{\Mset}{\mathsf M}

\renewcommand{\norm}[1]{\left\|#1\right\|}
\newcommand{\normtxt}[1]{\|#1\|}
\newcommand{\normmat}[1]{{\left\vert\kern-0.25ex\left\vert\kern-0.25ex\left\vert #1 
    \right\vert\kern-0.25ex\right\vert\kern-0.25ex\right\vert}}
\newcommand{\normmattxt}[1]{{\vert\kern-0.25ex\vert\kern-0.25ex\vert #1 
    \vert\kern-0.25ex\vert\kern-0.25ex\vert}}

\newcommand{\nset}{\mathbb N}
\newcommand{\Nset}{\mathsf N}

\newcommand{\PE}{\mathbb E}

\newcommand{\pscal}[2]{\langle #1,#2 \rangle}
\newcommandx{\posterior}[1][1=y]{p(#1|\mathscr{D})}
\newcommand{\PP}{\mathbb P}
\newcommand{\PQ}{\mathbb Q}
\newcommandx{\Psifn}[1][1=\alpha]{\Psi_{#1, L}}
\newcommandx{\Psif}[1][1=\alpha]{\Psi_{#1}}
\newcommandx{\Psifk}[3][1=\alpha, 2=N, 3=M]{\hat{\Psi}_{#1, #3, #2}}
\renewcommand{\Q}[3][]{\mathcal{Q}^{#1}_{#2}\lr{#3}}
\newcommand{\Qtxt}[3][]{\mathcal{Q}^{#1}_{#2}(#3)}
\newcommand{\tQ}[3][]{\tilde{\mathcal{Q}}^{#1}_{#2}\lr{#3}}
\newcommand{\tQtxt}[3][]{\tilde{\mathcal{Q}}^{#1}_{#2}(#3)}

\renewcommand{\R}[2]{\mathcal{R}_{#1}\lr{#2}}

\renewcommand{\rho}{\uprho}
\newcommand{\rhom}[1]{\varrho\lr{#1}}
\newcommand{\rhomtxt}[1]{\varrho(#1)}
\newcommand{\rhoinf}{\check{\uprho}_\star}
\newcommand{\rhosup}{\hat{\uprho}_\star}
\newcommand{\rhon}{\hat{\uprho}_n}
\newcommand{\rhotheta}{\hat{\uprho}_{\theta\theta'}}
\newcommand{\rhoinfsm}{\breve{\uprho}_\star}
\DeclareSymbolFont{tipa}{T3}{cmr}{m}{n}
\DeclareMathAccent{\invbreve}{\mathalpha}{tipa}{16}
\newcommand{\rhosupsm}{\invbreve{\uprho}_\star}
\newcommand{\ri}[1]{\mathrm{ri}(#1)}
\newcommand{\rmd}{\mathrm d}
\newcommand{\trmd}{\tilde{\rmd}}

\newcommand{\rmL}{\mathrm{L}}
\newcommand{\rmat}[2]{\mathbb{R}^{#1 \times #2}}

\newcommand{\rset}{\mathbb{R}}

\newcommand{\set}[2]{\lrc{#1\eqsp: \eqsp #2}}
\newcommand{\settxt}[2]{\{#1\eqsp: \eqsp #2\}}
\newcommand{\seq}[2][n\in\nset]{(#2)_{#1}}

\newcommand{\Spec}[1]{\mathrm{Spec}(#1)}
\newcommand{\Sset}{\mathsf{S}}
\newcommand{\tH}{\tilde{\mathsf{H}}}

\newcommand{\thv}{{\theta_\star}}
\newcommand{\thvv}{{\theta_{\star \star}}}
\newcommand{\ttheta}{\tilde{\theta}}
\newcommand{\tthv}{\tilde{\theta}_\star}

\newcommand{\Tthv}{{\mathsf{T}_{\star}}}

\newcommand{\Uset}{\mathsf U}

\newcommand{\Vset}{\mathsf V}
\newcommand{\Wset}{\mathsf W}
\newcommand{\Xset}{\mathsf X}
\newcommand{\xv}{x_\star}
\newcommand{\xvv}{x_{\star\star}}

\newcommand{\Yset}{\mathsf Y}

\newcommand{\Yk}{Y_{1:k}}

\newcommand{\zv}{\zeta_\star}

\newcommand{\pop}{\mathrm{pop}}
\newcommand{\samp}{\mathrm{samp}}

\begin{document}

\maketitle

\begin{abstract}
This paper introduces a general framework for iterative optimization algorithms and establishes under general assumptions that their convergence is asymptotically geometric.
We also prove that under appropriate assumptions, the rate of convergence can be lower bounded. The convergence is then only geometric, and we provide the exact asymptotic convergence rate.
This framework allows to deal with constrained optimization and encompasses the Expectation Maximization algorithm and the mirror descent algorithm, as well as some variants such as the $\alpha$-Expectation Maximization or the Mirror Prox algorithm.
Furthermore, we establish sufficient conditions for the convergence of the Mirror Prox algorithm, under which the method converges systematically to the unique minimizer of a convex function on a convex compact set.
\end{abstract}


\section{Introduction}

The minimization of a real-valued function is the most common formulation for mathematical optimization problems.  Examples of convex optimization problems in machine learning can be found for instance in \cite{bubeck}. 
For models  involving missing or latent data, \cite{dempster} introduced the modern formulation of the Expectation Maximization (EM) algorithm, whose convergence has been proved under general assumptions in \cite{wu}.

The asymptotic convergence rate of the EM algorithm has been widely
studied and identified as a ratio of missing information from the very
beginning \cite{dempster, meng_91,meng_93}. Since then, some links with gradient descent approaches have also
been drawn, see for instance \cite{lange}. Among the most notable
recent works, \cite{balakrishnan} provided quantitative results on the
non-asymptotic convergence of the EM algorithm to local optima by
considering smoothness and strong-concavity assumptions. In the
particular case of exponential families, \cite{kunstner} show that the
$M$-step is equivalent to a mirror descent update. This allows to
obtain non-asymptotic linear convergence rate, which directly depends on the ratio
of missing information. 

In this paper, instead of casting the EM algorithm into a gradient or
a mirror descent framework, we propose an extended formulation to
encompass both classes of algorithms, not restricted to exponential
families. 
Indeed, both EM and mirror descent algorithms can be defined using a bivariate function that is iteratively minimized with respect to one coordinate.
Such a representation can actually describe any iterative optimization algorithm whose minimization steps are parametrized only by the current parameter estimate. This paper provides the following contributions.
\begin{itemize}
    \item 
We prove under general assumptions that the convergence of such iterative optimization algorithms is asymptotically geometric, see Theorem~\ref{thm:main}. We also provide lower bounds for the rate of convergence, that allow to prove that the convergence can be only geometric, see Theorem~\ref{thm:main:lower_bound}, and in some cases to establish the exact asymptotic convergence rate, see Theorem~\ref{thm:main:exact_rate}.
We show that those assumptions are natural either in an EM or in a
mirror descent framework, and that they are satisfied generically
without requiring any notable technical work, in contrast with
non-asymptotic results that tend to be more demanding. Regarding the
EM algorithm, we retrieve the well-known ratio of missing information
under even more general assumptions, as the minimization mapping is
not required to be point-to-point and this framework allows to deal
with constrained optimization.
\item 
We derive results for settings with both finite and infinite data, as well as for a variant of the EM algorithm, known as the $\alpha$-EM algorithm, see \cite{matsuyama}. However, the most significant contribution is that brought to the mirror descent framework: under mild assumptions, we prove that its convergence is asymptotically geometric. This also applies to the mirror prox variant.
In a general manner, the convergence rates we exhibit are proved to be invariant to $C^2$-reparametrization.

\item  Furthermore, we prove that under general assumptions, the convergence of mirror prox is guaranteed for convex functions with a unique minimizer on a convex compact set, and that, without imposing any condition on the initialization.
\end{itemize}

This paper is organized as follows.
\Cref{sec:general_framework} introduces the general iterative optimization framework we consider and shows how it encompasses classical settings such as the EM algorithm or the mirror descent algorithm.
\Cref{sec:main} states the main general results of this paper on asymptotic convergence rates.
Sections \ref{sec:main:comm:submanifold}-\ref{sec:main:comm:thv} discuss the assumptions of \Cref{thm:main}, illustrating how they are met in those classical settings, but also in variants such as the $\alpha$-EM or the mirror prox algorithm.
\Cref{sec:main:proof} displays the proof of \Cref{thm:main}, and \Cref{sec:mirror_prox} is dedicated to the convergence of mirror prox.
A discussion follows in \Cref{sec:discussion}.
Additional proofs are postponed to \Cref{sec:app:proofs}, using technical results listed in \Cref{sec:app:tech} and proved in the Supplementary material.

\paragraph{Notation}
Throughout this paper, $\Spec{\cdot}$ denotes the spectrum of a matrix and $\rhom{\cdot}$ the spectral radius.
The Euclidean norm is denoted by $\norm{\cdot}_2$, the spectral norm by $\normmat{\cdot}_2$, the Frobenius norm by $\normmat{\cdot}_F$, and for all symmetric positive-definite matrices $S$, we define the norm $\norm{\cdot}_S$ by $\normtxt{x}_S^2 \eqdef x^\top S x$. The first derivative (resp. the second) of any univariate function $f$ is written $\partial f$ (resp. $\partial^2 f$). For all bivariate functions $\mcq\colon(x_1,x_2)\mapsto \Qtxt{x_1}{x_2}$ and $i,j\in\{1,2\}$, we write $\partial_i \mcq \eqdef \partial \mcq / \partial x_i$ and $\partial_{ij} \mcq \eqdef \partial^2 \mcq/\partial x_i \partial x_j$.
The maximum of two real numbers $a,b$ is denoted by $a\vee b$. For all topological spaces $\Eset$, their closure are written $\overline{\Eset}$ and their interior $\mathring{\Eset}$. Finally, $\Conv{\cdot}$ stands for convex hull, $\Aff{\cdot}$ for affine hull and $\ri{\cdot}$ for relative interior, i.e. the interior of a set within its affine hull.

\section{General framework} \label{sec:general_framework}

Let $q \in \nset^*$ and let $\mcq$ be a real-valued function defined on $\rset^q \times \rset^q$:
\begin{align*}
    \mcq \colon &\rset^q \times \rset^q \longrightarrow \rset\\
    & (\theta, \theta') \mapsto \Q{\theta}{\theta'}.
\end{align*}
Let $\Theta$ be a subset of $\rset^q$ and $\mcm$ be the point-to-set map defined on $\Theta$ by
\begin{equation*}
    \M{\theta} \eqdef \underset{\theta' \in \Theta}{\argmin} \eqsp \Q{\theta}{\theta'}.
\end{equation*}
In what follows, provided that $\M{\theta}\neq \emptyset$ for any $\theta
\in \Theta$, we let $\seq{\theta_n}$ be a sequence defined on $\Theta$ such that for all $n \in \nset$,
\begin{equation} \label{eq:general_framework:m}
    \theta_{n+1} \in \M{\theta_n}.
\end{equation}

\begin{manualexample}{1}[EM algorithm] \label{ex:em}
    Let $X$ and $Y$ be random variables taking values in measurable spaces $(\Xset,\mcx)$ and $(\Yset,\mcy)$, respectively.
    Assume that the pair $(X,Y)$ has a joint density function $p_{\thv}$ with respect to a reference measure $\mu$ on $\mcx\otimes \mcy$ that belongs to some parameterized family $\settxt{p_{\theta}}{\theta \in \Theta}$. Assume also that the state variable $X$ is
latent in the sense that the model is only partially observed through
the observation $Y$. In this case, the Expectation Maximization (EM)
algorithm, as defined in \cite[Appendix D.1, p.492]{randal_book}, provides an estimate of the unknown parameter $\thv$ by considering a sequence $\seq{\theta_n}$ defined on $\Theta$ by
    \begin{equation} \label{eq:general_framework:em_q}
        \theta_{n+1} \in \underset{\theta \in \Theta}{\argmin} \eqsp \Q{\theta_n}{\theta},
    \end{equation}
    where for all $\theta, \theta' \in \Theta \times \Theta$,
    \begin{equation} \label{eq:general_framework:em}        
        \Q{\theta}{\theta'} \eqdef - \PE_{\theta} \lrb{\log p_{\theta'}\lr{X, Y}|Y},
    \end{equation}
    and $\PE_{\theta}$ denotes the expectation under $p_{\theta}$.
    Note that $\mcq$ is a random function which depends on the observations we consider.
    For instance, in a model where $\seq[1\leq i\leq k]{X_i,Y_i}$ are independent and identically distributed, with $k$ observations $\seq[1\leq i\leq k]{Y_i}$, we define at the \textbf{sample level}: $X=(X_1,\ldots,X_k)$ and $Y=(Y_1,\ldots,Y_k)$, and inserting in \eqref{eq:general_framework:em}, we obtain up to a multiplicative constant (see \cite{balakrishnan}):
    \begin{equation}  \label{eq:general_framework:sample_em}
        \Q[\samp]{\theta}{\theta'} \eqdef - \frac 1 k \sum_{i=1}^k \int_{\Xset} p_{\theta}(x | Y_i) \log p_{\theta'}(x, Y_i) \mu(\rmd x).
    \end{equation}
    In the limit of infinite data (i.e. $k \to \infty$), we define at the \textbf{population level}:
    \begin{equation}   \label{eq:general_framework:pop_em}
        \Q[\pop]{\theta}{\theta'} \eqdef - \int_{\Yset} \lr{\int_{\Xset} p_{\theta}(x | y) \log p_{\theta'}(x, y) \mu(\rmd x)} p_{\thv}(y) \mu(\rmd y),
    \end{equation}
    where $x \mapsto p_{\theta}(x|y)$ denotes the conditional density of $X$ given $Y$ when the parameter value is $\theta$ and where we assume that  $\seq[i\geq 1]{Y_i}$ are iid with density $p_\thv$.
    Both settings are studied in this paper, replacing $\mcq$ in \eqref{eq:general_framework:em_q} by $\mcq^\samp$ or $\mcq^\pop$.
\end{manualexample}

\begin{manualexample}{2.1}[Mirror descent] \label{ex:mirror_descent}
    Let $\Cset$ be a convex compact set of $\rset^q$ and $f$ be a real-valued function defined on $\Cset$.
    The mirror descent strategy defined in \cite[Chapter 4, p.296]{bubeck} considers a convex open set $\Dset$ of $\rset^q$ such that $\Cset$ is contained in the closure of $\Dset$ and $\Cset \cap \Dset \neq \emptyset$,
    along with a mirror map $\Phi \colon \Dset \rightarrow \rset$, that is,
    \begin{enumerate}[(i)]
        \item $\Phi$ is strictly convex and differentiable, \label{hyp:mirror_map:strict_diff}
        \item the gradient of $\Phi$ takes all possible values: $\partial \Phi(\Dset) = \rset^q$, \label{hyp:mirror_map:surj}
        \item the gradient of $\Phi$ diverges on the boundary of $\Dset$: $\lim_{x \to \partial \Dset} \normtxt{\partial \Phi(x)} = +\infty$. \label{hyp:mirror_map:div}
    \end{enumerate}
    Then, the mirror descent algorithm produces two sequences $\seq{\theta_n}$ and $\seq{\zeta_n}$, defined on $\Cset$ and $\Dset$ respectively by
    \begin{align} \label{eq:ex:def_md}
        \left\{
        \begin{array}{ll}
            \partial \Phi(\zeta_{n+1}) = \partial \Phi(\theta_n) - \eta g_n, \quad \mbox{where} \ g_n \in \partial f(\theta_n), \\
            \theta_{n+1} \in \argmin_{\theta \in \Cset \cap \Dset} D_{\Phi}(\theta,\zeta_{n+1}),
        \end{array}
    \right.
    \end{align}
    where $\eta>0$ is the step-size, $\partial f$ is the sub-differential of $f$ (by abuse of notation) and $D_{\Phi}$ is the Bregman divergence associated with $\Phi$:
    \begin{equation} \label{eq:ex:bregman}
        \forall x,y \in \Dset, \quad D_{\Phi}(x,y) = \Phi(x) - \Phi(y) - \partial \Phi(y)^\top (x-y).
    \end{equation}
    Note that gradient descent is a particular case of mirror descent with $\Phi \colon x \mapsto x^\top x/2$.
    Following \cite[p.301]{bubeck}, mirror descent can be rewritten as
    \begin{equation*}
        \theta_{n+1} \in \underset{\theta \in \Cset \cap \Dset}{\argmin} \eqsp \eta g_n^\top \theta + D_{\Phi}(\theta,\theta_n),
    \end{equation*}
    which fits into the general framework \eqref{eq:general_framework:m} with $\Theta \eqdef \Cset \cap \Dset$ and $\mcq$ defined for all $(\theta, \theta') \in \Theta\times\Theta$ by
    \begin{equation} \label{eq:ex:md_q}
        \Q{\theta}{\theta'} \eqdef \eta g^\top \theta' + D_{\Phi}(\theta',\theta), \quad \mbox{where} \ g \in \partial f(\theta).
    \end{equation}
\end{manualexample}

\begin{manualexample}{2.2}[Mirror prox] \label{ex:mirror_prox}
    Mirror prox is a variant of mirror descent defined by the following equations \cite[Chapter 4, p.305]{bubeck}:
    \begin{align*}
        \partial \Phi(&\zeta'_{n+1}) = \partial \Phi(\theta_n) - \eta \partial f(\theta_n), \\
        &\zeta_{n+1} \in \underset{\theta \in \Cset \cap \Dset}{\argmin} \eqsp D_{\Phi}(\theta, \zeta'_{n+1}), \\
        \partial \Phi(&\theta'_{n+1}) = \partial \Phi(\theta_n) - \eta \partial f(\zeta_{n+1}), \\
        &\theta_{n+1} \in \underset{\theta \in \Cset \cap \Dset}{\argmin} \eqsp D_{\Phi}(\theta, \theta'_{n+1}).
    \end{align*}
    Straightforward algebra yields the equivalent definition:
    \begin{align}
        \zeta_{n+1} \in \M{\theta_n} = &\eqsp \underset{\theta \in \Cset \cap \Dset}{\argmin} \eqsp \eta \partial f(\theta_n)^\top \theta + D_{\Phi}(\theta, \theta_n), \label{eq:ex:mp:zeta} \\
        \theta_{n+1} \in &\eqsp \underset{\theta \in \Cset \cap \Dset}{\argmin} \eqsp \eta \partial f(\zeta_{n+1})^\top \theta + D_{\Phi}(\theta, \theta_n), \nonumber
    \end{align}
    where $\mcm$ is defined for mirror descent by \eqref{eq:ex:md_q}. We deduce from \Cref{ex:suff:md} that mirror prox fits into the general framework \eqref{eq:general_framework:m} with $\Theta \eqdef \Cset \cap \Dset$ and $\mcq^m$ defined on $\Theta \times \Theta$ by
    \begin{equation} \label{eq:ex:mp_q}
        \Q[m]{\theta}{\theta'} \eqdef \eta \partial f\lr{\M{\theta}}^\top \theta' + D_{\Phi}(\theta', \theta).
      \end{equation}
The fact that $\M{\theta}$  is a singleton is ensured by (i) and (iii) as in this case $\Phi$ is a Legendre function, see \cite[Lemma~11.1]{cesa2006prediction} or \cite[Theorem 3.12]{bauschke}.
\end{manualexample}

\section{Asymptotic convergence rate} \label{sec:main}

Assume there exists $\thv \in \Theta$ such that $\partial_2 \mcq$ is well-defined in a neighborhood of $\thv$ and differentiable at $(\thv,\thv)$, and write
\begin{equation*}
    \Av \eqdef \dtwice{2}{2}{\Q{\thv}{\thv}}, \quad \Bv \eqdef -\dtwice{1}{2}{\Q{\thv}{\thv}}.
\end{equation*}
Let $\Vset \eqdef \mathrm{span}\settxt{\theta-\theta'}{\theta,\theta'\in\Theta}$ be the direction of $\Aff{\Theta}$.
Consider the following set of assumptions.
\begin{hyp}{H}{H} \item \label{hyp:main:convexity}
    The set $\Theta$ is convex.
\end{hyp}
\begin{hyp}{H}{H} \item \label{hyp:main:convergence}
    The sequence $\seq{\theta_n}$ converges to $\thv$.
\end{hyp}
\begin{hyp}{H}{H} \item \label{hyp:main:regularity}
    There exists a neighborhood of $(\thv, \thv)$ on which $\mcq$ is continuous and $\partial_2 \mcq$ is well-defined and $C^1$-differentiable.
\end{hyp}
\begin{hyp}{H}{H} \item \label{hyp:main:thv}
    The matrix $\Bv$ is symmetric and for all $v\in\Vset\setminus \{0\}$, $v^\top \Av v > |v^\top \Bv v|$. 
\end{hyp}
Under \ref{hyp:main:thv}, we can define
\begin{equation} \label{eq:main:rhosup}
    \rhosup \eqdef \sup_{v\in\Vset \setminus \{0\}} \frac{|v^\top \Bv v|}{v^\top \Av v}.  
  \end{equation}
In what follows, we set by convention, $\log 0 = -\infty$.

\begin{theorem} \label{thm:main}
    Assume that \ref{hyp:main:convexity}-\ref{hyp:main:thv} hold. Then, $\rhosup \in[0;1)$ and for all $\rho \in (\rhosup; 1)$,
    \begin{equation*}
        \theta_n-\thv = \lito(\rho^n), 
    \end{equation*}
    or equivalently,
\begin{equation*}
        \underset{n\to\infty}{\lim \sup} \eqsp \frac 1 n \log\norm{\theta_n-\thv}_2 \leq \log\rhosup.
    \end{equation*}
  \end{theorem}
As any two norms on a finite-dimensional linear space are equivalent, Theorem~\ref{thm:main} and the next results could be given using another norm on $\Theta$. They are stated here with $\|\cdot\|_2$ for simplicity.

\begin{proof} \renewcommand{\qedsymbol}{}
    See \Cref{sec:main:proof}.
  \end{proof}
In \cite[Theorem~1]{balakrishnan}, the authors prove that  the population EM algorithm converges geometrically. Their proof rely mainly on convergence results for gradient ascent algorithms applied to the intermediate quantity of the EM algorithm which is assumed to be smooth and strongly concave.  \Cref{thm:main} establishes under general assumptions that the convergence of the algorithms introduced in \Cref{sec:general_framework} is asymptotically geometric.
Corollary~\ref{cor:main:Q} extends the statement of \Cref{thm:main} to the values taken by the function $(\theta,\theta') \mapsto \mcq_\theta(\theta')$ which are invariant to the choice of the parametrisation.

\begin{corollary}\label{cor:main:Q}
    Under \ref{hyp:main:convexity}-\ref{hyp:main:thv}, if $\partial_1 \Q{\thv}{\thv}$ is well-defined, then for all $\rho \in (\rhosup; 1)$,
    \begin{equation*}
        \Q{\theta_n}{\theta_{n+1}} - \Q{\thv}{\thv} = \lito(\rho^{n}).
    \end{equation*}
\end{corollary}
\begin{proof} \renewcommand{\qedsymbol}{}
    See \Cref{sec:app:main}.
\end{proof}
Besides, the speed of convergence can be lower-bounded if the limit $\thv$ lies in the relative interior of $\Theta$.
Under \ref{hyp:main:thv}, we can define
\begin{equation} \label{eq:main:rhoinf}
    \rhoinf \eqdef \inf_{v\in\Vset \setminus \{0\}} \frac{|v^\top \Bv v|}{v^\top \Av v}.
\end{equation}

\begin{theorem} \label{thm:main:lower_bound}
    Assume that \ref{hyp:main:convexity}-\ref{hyp:main:thv} hold, that $\thv\in\ri{\Theta}$, and that the sequence $\seq{\theta_n}$ is not eventually equal to $\thv$. Then, $\rhoinf\in[0;1)$ and
    \begin{equation*}
        \log \rhoinf \leq \underset{n\to\infty}{\lim \inf} \eqsp \frac 1 n \log\norm{\theta_n-\thv}_2.
    \end{equation*}
\end{theorem}
\begin{proof} \renewcommand{\qedsymbol}{}
    See \Cref{sec:app:main}.
\end{proof}
If the limit $\thv$ lies in the relative interior of $\Theta$ and $\rhoinf > 0$, the asymptotic convergence is therefore only geometric.

\begin{theorem} \label{thm:main:exact_rate}
    Assume that \ref{hyp:main:convexity}-\ref{hyp:main:thv} hold, that $\partial_2 \mcq$ is $C^2$-differentiable in a neighborhood of $(\thv, \thv)$, that $\thv\in\ri{\Theta}$ and that for all $p\in \nset$, $\mathrm{Span}(\theta_{n}-\thv, n\geq p)=\Vset$. Then $\rhoinf, \rhosup\in[0;1)$ and
    \begin{equation*}
        \log \min\lr{\rhosup, \frac {\rhoinf}{\rhosup}} \leq \underset{n\to\infty}{\lim \inf} \eqsp \frac 1 n \log\norm{\theta_n-\thv}_2,
    \end{equation*}
    where in the left-hand term we use the convention $0/0=0$ and $\log(0)=-\infty$. In particular, if $\rhosup^2 \leq \rhoinf$ then
    \begin{equation*}
        \underset{n\to\infty}{\lim} \frac 1 n \log \norm{\theta_n - \thv}_2 = \log \rhosup.
    \end{equation*}
\end{theorem}
\begin{proof} \renewcommand{\qedsymbol}{}
    See \Cref{sec:app:main}.
\end{proof}

\section{Comments on H1} \label{sec:main:comm:submanifold}

The assumption that $\Theta$ needs to be convex can be relaxed as follows.
\newcounter{hypH}
\begin{hyp}{H'}{Hprime} \item \label{hyp:sm:theta}
    There exist $\Eset\subset\rset^q$ and a submanifold $\Sset\subset\rset^q$ of class $C^2$ such that $\Theta = \Eset\cap\Sset$ and $\thv\in\mathring{\Eset}$.
\end{hyp}
Under \ref{hyp:sm:theta}, if $d$ is the dimension of the submanifold $\Sset$, for all $x\in\Sset$, there exist $\Uset_1$, $\Uset_2$ two open neighborhoods of $x$ and the null-vector $\0$ in $\rset^q$, respectively, and a $C^2$-diffeomorphism $\psi\colon\Uset_1\rightarrow\Uset_2$ such that $\psi(x)=\0$ and $\psi(\Uset_1\cap\Sset)=\Uset_2\cap(\rset^d\times\{0\}^{q-d})$. Note that we identify $\rset^q$ and $\rset^d\times\rset^{q-d}$ in a standard way using $(x_1,\ldots,x_q)\mapsto ((x_1,\ldots,x_d),(x_{d+1},\ldots,x_q))$. Write $\Tthv$ the tangent space to $\Sset$ at the point $\thv$. If $\Uset_1$ and $\Uset_2$ are two open neighborhoods of $\thv$ and the null-vector $\0$ in $\rset^q$, respectively, and $\psi\colon\Uset_1\rightarrow\Uset_2$ is a $C^1$-diffeomorphism  such that $\psi(\thv)=\0$ and $\psi(\Uset_1\cap\Sset)=\Uset_2\cap(\rset^d\times\{0\}^{q-d})$, then $\Tthv = \partial\psi^{-1}_{\thv}(\rset^d\times\{0\}^{q-d})$, where $\partial\psi$ is the differential of $\psi$ at $\thv$. 
\begin{enumerate}[label=\textbf{\sf(H'\arabic*)},resume=hypHprime, start=4]
  \begin{sf}
    \item \label{hyp:sm:thv}
    The matrix $\Bv$ is symmetric and for all $v\in\Tthv\setminus \{0\}$, $v^\top \Av v > |v^\top \Bv v|$.
    \end{sf}
\end{enumerate}
Under \ref{hyp:sm:thv}, we can define
\begin{equation} \label{eq:sm:rhos}
    \rhoinfsm \eqdef \inf_{v\in\Tthv\setminus \{0\}} \frac{|v^\top \Bv v|}{v^\top \Av v} \quad \mbox{and} \quad \rhosupsm \eqdef \sup_{v\in\Tthv \setminus \{0\}} \frac{|v^\top \Bv v|}{v^\top \Av v}.
\end{equation}

\begin{theorem} \label{thm:sm}
    Assume that \ref{hyp:sm:theta}, \ref{hyp:main:convergence}, \ref{hyp:main:regularity} and \ref{hyp:sm:thv} hold. Then, $\rhosupsm\in[0;1)$ and for all $\rho\in(\rhosupsm;1)$,
    \begin{equation*}
        \theta_n-\thv = \lito(\rho^n).
    \end{equation*}
    Furthermore, if the sequence $\seq{\theta_n}$ is not eventually equal to $\thv$, then $\rhoinfsm\in[0;1)$ and for all $\rho\in(0;\rhoinfsm)$,
    \begin{equation*}
        \rho^n = \lito\left(\norm{\theta_n-\thv}_2\right).
    \end{equation*}
  \end{theorem}

\begin{proof}
    See \Cref{sec:app:sm}.
\end{proof}

\begin{remark}
If \ref{hyp:main:convexity} holds with $\thv\in\ri{\Theta}$, then
\ref{hyp:sm:theta} is satisfied with $\Eset\eqdef \Theta + \Vset^\perp$ and $\Sset\eqdef\Aff{\Theta}$.
\end{remark}

\begin{remark}
If $\thv$ does not lie in the relative interior of $\Theta$, the asymptotic convergence rates are not necessarily invariant to $C^2$-reparametrization. Define for instance $\tilde{\mcq}_{\theta_0} (\theta_1)= \theta_0^2
  -\theta_0 \theta_1 +\theta_1^2$ with $\Theta=[1;+\infty[$. Using the
  reparametrization function $\Psi:\theta \mapsto \theta^\alpha$, we
  set $\check{\mcq}_{\theta_0} (\theta_1)=\tilde{\mcq}_{\Psi (\theta_0)} (\Psi
  (\theta_1))=\theta_0^{2\alpha} -\theta_0^\alpha \theta_1^\alpha +
  \theta_1^{2\alpha}$. Then, with $\mcq=\tilde{\mcq}$, we get $\rhoinf =
  \rhosup = 1/2$ whereas with $\mcq=\check{\mcq}$ and $\alpha=2/5$, we get $\rhoinf =
  \rhosup = 2$.
\end{remark}

\section{Comments on H2} \label{sec:main:comm:suff}

The convergence of the sequence $\seq{\theta_n}$ to $\thv$ (stated in
\ref{hyp:main:convergence}) may be the most challenging assumption of \Cref{thm:main}.
However, we provide alternative sufficient assumptions to establish
such convergence.

\begin{hyp}{$\mathsf{H}$2.}{Htilde} \item \label{hyp:suff:compacity}
    The set $\Theta$ is compact.
\end{hyp}

\begin{hyp}{$\mathsf{H}$2.}{Htilde} \item \label{hyp:suff:continuity}
    The function $\mcq$ is continuous on $\Theta\times\Theta$.
\end{hyp}

\begin{hyp}{$\mathsf{H}$2.}{Htilde} \item \label{hyp:suff:accu_point}
    The point $\thv$ is a limit point of the sequence $\seq{\theta_n}$.
\end{hyp}
\begin{hyp}{$\mathsf{H}$2.}{Htilde} \item \label{hyp:suff:m_thv}
    $\M{\thv}=\{\thv\}$.
\end{hyp}

\begin{theorem} \label{thm:suff:convergence}
    Under \ref{hyp:main:convexity},
    \ref{hyp:suff:compacity}-\ref{hyp:suff:m_thv}, \ref{hyp:main:regularity} and \ref{hyp:main:thv},
    the sequence $\seq{\theta_n}$ converges to $\thv$.
\end{theorem}
\begin{proof}
See \Cref{sec:app:comments}.
\end{proof}

\begin{remark}
 Assumption \ref{hyp:suff:accu_point} weakens \ref{hyp:main:convergence} by only requiring that \ref{hyp:main:regularity}-\ref{hyp:main:thv} hold for an arbitrary $\thv$ in the limit set of $\seq{\theta_n}$, which is non-empty under \ref{hyp:suff:compacity}.
\end{remark}

\begin{manualexample}{2.1}[Mirror descent, cont.] \label{ex:suff:md}
    The map $\mcm$ is point-to-point on $\Theta$ under the assumptions of the definition (see \Cref{ex:mirror_descent} in page \pageref{ex:mirror_descent}).
    Indeed, the surjectivity of the gradient in \ref{hyp:mirror_map:surj} provides the existence of $\zeta_{n+1}$ in \eqref{eq:ex:def_md}, and the strict convexity of $\Phi$ in \ref{hyp:mirror_map:strict_diff} proves its uniqueness.
    Assumptions \ref{hyp:mirror_map:strict_diff} and \ref{hyp:mirror_map:div} ensure the existence and the uniqueness of $\theta_{n+1}$ in \eqref{eq:ex:def_md} (see \cite[Theorem 3.12]{bauschke}).
    Note that if $f$ is convex or differentiable on $\Cset$, then for all $\theta\in\Cset$, $\partial f(\theta) \neq \emptyset$ and $g_n$ can be defined in \eqref{eq:ex:def_md}.

    Moreover, if $\thv$ is a local minimizer of $f$ and $f$ is differentiable at $\thv$, then \ref{hyp:suff:m_thv} is met.
    Indeed, those two assumptions provide that for all $\theta\in\Theta$, $\partial f(\thv)^\top (\theta - \thv) \geq 0$, and thus $\Q{\thv}{\theta} \geq \Q{\thv}{\thv}$ in \eqref{eq:ex:md_q} with equality if and only if $\theta = \thv$. 
\end{manualexample}

\Cref{prop:suff:mp} establishes that the mirror prox approach described in  \Cref{ex:mirror_prox} statisfies \ref{hyp:suff:accu_point} and \ref{hyp:suff:m_thv} under additional assumptions.

\begin{proposition} \label{prop:suff:mp}
    Assume in \Cref{ex:mirror_prox} that \ref{hyp:suff:compacity}-\ref{hyp:suff:continuity} hold and that: 
    \hypertarget{hyp:ex:mp:diff_2}{(i)} $\Phi$ and $f$ are twice differentiable on $\Theta$,
    \hypertarget{hyp:ex:mp:cesaro}{(ii)} $\Phi$ is $\gamma$-strongly convex on $\Cset\cap\Dset$ and $f$ is convex and $\beta$-smooth, with respect to $\norm{\cdot}_2$,
    \hypertarget{hyp:ex:mp:eta}{(iii)} $\eta \in (0;\gamma/\beta)$,
    \hypertarget{hyp:ex:mp:mini}{(iv)} $\thv$ is the unique minimizer of $f$ on $\Cset$,
    \hypertarget{hyp:ex:mp:ri}{(v)} $\thv\in\ri{\Theta}$.
    Then, \ref{hyp:suff:accu_point} and \ref{hyp:suff:m_thv} hold.
\end{proposition}

\begin{proof}
    See \Cref{sec:app:comments}.
\end{proof}

We also provide alternative assumptions to prove \ref{hyp:suff:m_thv} in the general case.

\begin{hyp}{$\tH$4.}{tildeH4} \item \label{hyp:suff:singleton}
    $\Mtxt{\thv}$ is a singleton.
\end{hyp}

\begin{hyp}{$\tH$4.}{tildeH4} \item \label{hyp:suff:lyapunov}
    There exists a continuous function $\vartheta \colon \Theta \rightarrow \rset$ such that for all $\theta \in \Theta$, $\theta'\in\M{\theta}$,
    \begin{equation*}
        \vartheta(\theta') \leq \vartheta(\theta),
    \end{equation*}
    with equality if and only if $\theta = \theta'$.
\end{hyp}

\begin{theorem} \label{thm:suff:m_thv}
   Assume \ref{hyp:suff:compacity},
  \ref{hyp:suff:continuity}, \ref{hyp:suff:accu_point}. Then, \ref{hyp:suff:singleton} and \ref{hyp:suff:lyapunov} imply \ref{hyp:suff:m_thv}.
  \end{theorem}

\begin{proof}
    See \Cref{sec:app:comments}.
\end{proof}

\begin{manualexample}{1}[EM algorithm, cont.]
    By definition of the intermediate quantity \eqref{eq:general_framework:em}, the  EM algorithm monotonically increases the likelihood of the observations and Assumption \ref{hyp:suff:lyapunov} is satisfied as soon as the log-likelihood is continuous, see for example \cite[Proposition 10.1.4, p.350]{inference_hmm}.
\end{manualexample}

\section{Comments on H4} \label{sec:main:comm:thv}

First of all, the matrix $\Bv$ appears to be symmetric in all the examples below. The discussion then focuses on the domination assumption and on the value of the convergence rate $\rhosup$.
Note that the domination assumption in \ref{hyp:main:thv} is equivalent to having both $\tAv \succ \tBv$ and $\tAv \succ -\tBv$. In the case where $\tAv \succ 0$, it is equivalent to $\rhosup\in[0;1)$.

\begin{manualexample}{1.1}[Population EM] \label{ex:main:pop_em}
    Assume that $\thv$ is the true parameter of the model, that for all $x,y \in \Xset, \Yset$, the functions $\theta \mapsto p_{\theta}(x|y)$ and $\theta \mapsto p_{\theta}(y)$ are twice differentiable in a neighborhood of $\thv$, and that conditions similar to \cite[Assumption AD.1, p.492]{randal_book} hold to differentiate under the integral sign. Then, we prove in \Cref{sec:app:thm:main}, see \eqref{eq:ex:main:dtwice22:pop} and \eqref{eq:ex:main:dtwice12:pop},  that
    \begin{equation} \label{eq:main:comm:a_b_pop_em}
        \Av^\pop = I_{X,Y}(\thv) \quad \mbox{and} \quad \Bv^\pop = I_{X,Y}(\thv) - I_Y(\thv),
    \end{equation}
    where $I_{X,Y}(\theta) \eqdef -\PE_{\theta} [\partial^2_\theta \log p_{\theta}(X,Y)]$ and $I_Y \eqdef -\PE_{\theta} [\partial^2_\theta \log p_{\theta}(Y)]$ denote the Fisher information matrices of $(X,Y)$ and $Y$, respectively.
    Therefore, \ref{hyp:main:thv} is satisfied as soon as $I_Y(\thv) \succ 0$. Regarding the value of $\rhosup$, the above expressions of $\Av^\pop$ and $\Bv^\pop$ provide the well-known ratio of missing information $I_{X,Y}(\thv)^{-1} I_{X|Y}(\thv)$ (see \cite{dempster, kunstner,meng_91,meng_93,orchard}), where 
$$
I_{X|Y}(\thv) = \int_{\Yset} \int_{\Xset} p_{\thv}(y) p_{\theta}(x | y) \partial \log p_{\theta'}(x| y) \lrb{\partial \log p_{\theta}(x | y)}^\top \mu(\rmd x) \mu(\rmd y).
$$
\end{manualexample}

\begin{manualexample}{1.2}[Sample EM] \label{ex:main:sample_em}
    As for the other examples, all the results below are proved in \Cref{sec:app:thm:main}.
    Assume that for all $x,y \in \Xset, \Yset$, the functions $\theta \mapsto p_{\theta}(x|y)$ and $\theta \mapsto p_{\theta}(y)$ are twice differentiable in a neighborhood of $\thv$, and that conditions similar to \cite[Assumption AD.1, p.492]{randal_book} hold to differentiate under the integral sign.
    Let $\seq[i\in\nset^*]{Y_i}$ be a sequence of independent and identically distributed random variables with probability density function $p_{\thv}$, and write for all $k\in\nset^*$, $Y_{1:k}\eqdef\seq[1\leq i\leq k]{Y_i}$. Then, for all $k\in\nset^*$, by \eqref{eq:ex:main:dtwice22:sample} and \eqref{eq:ex:main:dtwice12:sample},
    \begin{align*}
        \Av^\samp\lr{Y_{1:k}} &= \frac 1 k \sum_{i=1}^k \lr{I_{X|Y=Y_i}(\thv) - \partial^2 \log p_{\thv}(Y_i)}, \\
        \Bv^\samp\lr{Y_{1:k}} &= \frac 1 k \sum_{i=1}^k I_{X|Y=Y_i}(\thv),
    \end{align*}
where $I_{X|Y=Y_i}(\thv) =  \int_{\Xset} p_{\thv}(x | Y_i) \partial \log p_{\thv}(x| Y_i) \lrb{\partial \log p_{\thv}(x | Y_i)}^\top \mu(\rmd x)$.  

Note that $\Av^\samp\lr{Y_{1:k}}$ and $\Bv^\samp\lr{Y_{1:k}}$ converges almost  surely to $\Av^\pop$ and $\Bv^\pop$. Then, if the corresponding population EM meets \ref{hyp:main:thv}, almost surely, for sufficiently large $k$, the sample EM meets \ref{hyp:main:thv}.
    Denoting by $\rhosup^\pop$ and $\rhosup^\samp(Y_{1:k})$ their respective rates, as defined in \eqref{eq:main:rhosup}, by Lemma~\ref{lem:rhopop_conv}, we also have that
    \begin{equation} \label{eq:main:comm:slln_rho}
        \rhosup^\samp(Y_{1:k}) \overset{a.s.}{\longrightarrow} \rhosup^\pop.
    \end{equation}
    Furthermore, if $\partial^2 \log p_\thv(X_1, Y_1), \partial^2 \log p_{\thv}(Y_1) \in \rmL^2(\rset^{q\times q})$, Lemma~\ref{lem:rhopop_conv} also establishes that for all $\delta\in(0;1)$ there exists $C_\delta>0$ such that
    \begin{equation} \label{eq:main:comm:clt_rho}
        \underset{k\rightarrow\infty}{\lim \inf} \eqsp \PP\lr{\lrav{\rhosup^\samp(Y_{1:k}) - \rhosup^\pop} \leq \frac{C_\delta} {\sqrt{k}}} \geq 1-\delta.
    \end{equation}
\end{manualexample}

\begin{manualexample}{2.1}[Mirror descent, cont.] \label{ex:main:mirror_descent}
    If $f$ and $\Phi$ are twice differentiable in a neighborhood of $\thv$, we prove in \Cref{sec:app:thm:main} that
    \begin{equation} \label{eq:main:comm:md}
        \Av = \partial^2 \Phi(\thv) \quad \mbox{and} \quad \Bv = \partial^2 \Phi(\thv) - \eta \partial^2 f(\thv).
    \end{equation}
    If for all $v\in\Vset$, $v^\top \partial^2 f(\thv) v > 0$, the condition $\tAv \succ \tBv$ is automatically satisfied. The domination assumption in \ref{hyp:main:thv} then reduces to $\tAv \succ -\tBv$, which corresponds to $\eta$ being small enough.
    In the particular case of unconstrained gradient descent where $\Aff{\Theta}=\rset^q$ and $\Phi\colon x \mapsto x^\top x/2$, as \eqref{eq:main:comm:md} yields $\Av=I_q$ the above condition is equivalent to $\eta \in (0;2/\beta_\star)$, the optimal choice being $\eta = 2/(\alpha_\star+\beta_\star)$ where $\alpha_\star \eqdef \min \Spec{\partial^2 f(\thv)}$ and $\beta_\star \eqdef \max \Spec{\partial^2 f(\thv)}$.

    Besides, the asymptotic convergence rate $\rhosup$ can be interpreted similarly to the EM framework. Despite not being, strictly speaking, a ratio of missing information, $\rhosup$ still compares the mirror map $\Phi$ with the objective function $f$.
    Intuitively, the choice of a mirror map with variations closer to those of $f$ provides a better convergence rate.
    If $\eta=1$, the extreme case $\Phi = f$ yields $\Bv = 0$ and $\rhosup=0$, which is coherent with the fact that, in this case, the mirror descent is defined for all $n \in \nset$ by $\theta_{n+1} \in \argmin_{\theta \in \Theta} f(\theta)$.
\end{manualexample}

The following discussion extends the above interpretation to a general class of functions $\mcq$ that encompasses both mirror descent and the EM algorithm.
The first thing to note is that in both settings the function $\mcq$ can be redefined as
\begin{equation} \label{eq:main:comm:f_D}
    \Q{\theta}{\theta'} \eqdef f(\theta') + D(\theta, \theta'),
\end{equation}
where $f \colon \rset^q \rightarrow \rset$ is the objective function and $D \colon \rset^q \times \rset^q \rightarrow \rset$ is a function such that for all $\theta \in \Theta$, $\donce D(\theta, \theta) = 0$.
Indeed, it is common knowledge that the intermediate quantity of the EM algorithm can be expressed as
\begin{equation*}
    Q(\theta, \theta') = \log p_{\theta'}(Y) - D_{\mathrm{KL}}(p_{\theta}(X|Y) || p_{\theta'}(X|Y)),
\end{equation*}
where $D_{\mathrm{KL}}$ denotes the Kullback-Leibler divergence (see \cite{daudel} for example).
Regarding mirror descent, if $f$ is twice differentiable in \Cref{ex:mirror_descent}, straightforward computation yields the following equivalent definition for $\mcq$:
\begin{equation*}
    \Q{\theta}{\theta'} \eqdef f(\theta') + \frac 1 {\eta} D_{\Phi - \eta f}(\theta', \theta),
\end{equation*}
where the expression of $D_{\Phi - \eta f}$ follows that of \eqref{eq:ex:bregman} (and defines a Bregman divergence if $\Phi - \eta f$ is strictly convex).
Besides, the condition $\donce D(\theta, \theta) = 0$ for all $\theta \in \Theta$ is equivalent to $\donce{\Qtxt{\theta}{\theta}} = \partial f(\theta)$ for all $\theta \in \Theta$, hence
\begin{align*}
    \Av &= \dtwice{2}{2}{\Q{\thv}{\thv}}, \\
    \Bv &= - \dtwice{1}{2}{\Q{\thv}{\thv}} = \dtwice{2}{2}{\Q{\thv}{\thv}} - \partial^2 f(\thv),
\end{align*}
and
\begin{equation*}
    \rhosup = \sup_{v\in\Vset} \frac{|v^\top \lr{\dtwice{2}{2}{\Q{\thv}{\thv}} - \partial^2 f(\thv)} v|}{v^\top \dtwice{2}{2}{\Q{\thv}{\thv}} v}.
\end{equation*}
In the framework of \eqref{eq:main:comm:f_D}, the convergence rate can thus be viewed as a relative difference between the second-order variations of $\mcq$ and $f$.
Computing iteratively $\argmin_{\Theta} \Qtxt{\theta_n}{\cdot}$ to estimate $\argmin_{\Theta} f$ can prove useful if those minimizations are easier to carry out, but the price to pay in terms of iterations (through the convergence rate) is directly related to how far the surrogate function $\mcq$ is from the objective function $f$.
If $\tAv$ is invertible, $\rhosup$ is indeed the spectral radius of $\tBv\tAv^{-1} = (\dtwice{2}{2}{\tQtxt{\thv}{\thv}} - \partial^2 \tilde{f}(\thv)) (\dtwice{2}{2}{\tQtxt{\thv}{\thv}})^{-1}$ by \Cref{lemma:tech:rho_a_b}, and the interpretation of a {\em ratio of missing information} generalizes to that of a ratio measuring the loss of exactness in the minimization procedure. 

Finally, in the particular case where $D$ is a distance or a divergence twice differentiable at $(\thv, \thv)$ with respect to the second argument, $\theta \in \argmin_\Theta D(\theta, \cdot)$ for all $\theta\in\Theta$ implies $\Bv = \dtwice{2}{2}{D\lr{\thv, \thv}} \succeq 0$.
The domination assumption in \ref{hyp:main:thv} then boils down to $\partial^2 \tilde{f}(\thv) \succ 0$.

\begin{manualexample}{1.3}[The $\alpha$-EM algorithm] \label{ex:alpha_em}

The above discussion highlighted how the choice of the surrogate function $\mcq$ determines the convergence rate $\rhosup$.
In the EM algorithm of \Cref{ex:em}, where the function $\mcq$ can be defined for all $\theta, \theta' \in \Theta \times \Theta$ as
\begin{equation*}
    \Q{\theta}{\theta'} \eqdef - \log p_{\theta'}(Y) - \int_{\Xset} p_{\theta}(x | Y) \log \frac {p_{\theta'}(x | Y)} {p_{\theta}(x | Y)} \mu(\rmd x),
\end{equation*}
the question then rises whether replacing the Kullback-Leibler divergence by an $\alpha$-divergence (see \cite{daudel} for example) could provide a better convergence rate.
This leads to replacing the previous expression of $\mcq$ by:
\begin{equation} \label{eq:alpha_em:q_alpha}
    \Q[\alpha]{\theta}{\theta'} \eqdef -\int_{\Xset} p_{\theta}(x | Y) f_{\alpha} \lr{\frac {p_{\theta'}(x,Y)} {p_{\theta}(x,Y)}} \mu(\rmd x),
\end{equation}
where for all $\alpha\in\rset\setminus\{0,1\}$, the concave function $f_\alpha$ is defined on $\rset_+^*$ by $f_\alpha(x) \eqdef (1-x^{\alpha})/\alpha(\alpha-1)$ and $f_0 \eqdef \log$.
This approach has been introduced and developed in \cite{matsuyama}.
We provide further elements for the choice of $\alpha$ by proving (see \Cref{sec:app:thm:main}) that at a population level, under the assumptions of \Cref{ex:main:pop_em} with $f_\alpha$ instead of $f_0$,
\begin{equation} \label{eq:alpha_em:a_b_star}
    \Av^{\alpha} = I_{X, Y}(\thv) \quad \mbox{and} \quad \Bv^{\alpha} = I_{X,Y}(\thv) -\frac 1 {1-\alpha} I_Y(\thv).
\end{equation}
Note that when $\alpha=0$ we recover the previous quantities $\Av$ and $\Bv$ for the classical EM algorithm.
If $I_Y(\thv)\succ 0$, then $\alpha \in(0;1/2)$ is a sufficient condition to meet \ref{hyp:main:thv}.
Besides, if $\Av^{\alpha}$ is invertible we can write
\begin{equation} \label{eq:alpha_em:rho_alpha}
    \lr{\Av^{\alpha}}^{-1} \Bv^{\alpha} = I_{X, Y}(\thv)^{-1}I_{X|Y}(\thv) - \frac{\alpha}{1-\alpha} I_{X, Y}(\thv)^{-1}I_Y(\thv).
\end{equation}
A necessary condition to improve the convergence rate is then $\alpha/(1-\alpha) > 0$, i.e. $\alpha \in (0;1)$.
We can also rewrite \eqref{eq:alpha_em:rho_alpha} as follows:
\begin{equation*}
    \lr{\Av^{\alpha}}^{-1} \Bv^{\alpha} = \frac 1 {1-\alpha} \Av^{-1} \Bv - \frac{\alpha}{1-\alpha} I_q.
\end{equation*}
By the positivity of $\Bv$ for the original EM algorithm, we deduce that the optimal choice of $\alpha$ corresponds to $\alpha = (\rhosup + \rhoinf)/2$ and $\rhosup^{\alpha} = (\rhosup-\rhoinf)/(2-\rhosup-\rhoinf)$, where $\rhosup$ and $\rhoinf$ are defined in \eqref{eq:main:rhosup} and \eqref{eq:main:rhoinf} for the classical EM algorithm.

As a remark, we can see in \cite{matsuyama} that the $\alpha$-EM algorithm does not simply change the $D$-function in \eqref{eq:main:comm:f_D}, it also replaces the objective function with a different bivariate function.

\end{manualexample}

\begin{manualexample}{2.2}[Mirror prox, cont.] \label{ex:main:mirror_prox}
    Assume that \ref{hyp:suff:compacity} hold, that $\Phi$ and $f$ are $C^1$-differentiable on $\Theta$ and twice differentiable at $\thv$, and that the corresponding mirror descent satisfies \ref{hyp:main:regularity}-\ref{hyp:main:thv} and $\thv=\M{\thv}\in\ri{\Theta}$.
    Then, we prove in \Cref{sec:app:thm:main} that
    \begin{equation*}
        \tAv^m = \tAv \quad \mbox{and} \quad \tBv^m = \tAv + \tBv\tAv^{-1}\tBv - \tBv,
    \end{equation*}
    where $\Av$, $\Bv$ are defined in \eqref{eq:main:comm:md} for mirror descent.
    This provides the symmetry of $\tBv^m$ and thus of $\Bv^m$, as well as
    \begin{equation} \label{eq:main:comm:mp}
        (\tAv^m)^{-1}\tBv^m = (\tAv^{-1}\tBv)^2 - \tAv^{-1}\tBv + I_d.
    \end{equation}
    We deduce that under \ref{hyp:main:thv} for mirror descent, $\rhosup^m < 1$ if and only if $\tBv \succ 0$, which is met as soon as $\eta\in(0;\gamma_\star/\beta_\star)$, where $\beta_\star \eqdef \max \Spec{\partial^2 \tilde{f}(\thv)}$ and $\gamma_\star \eqdef \min \Spec{\partial^2 \tilde{\Phi}(\thv)}$. 
    Besides, a sufficient condition for the $C^1$-differentiability of $\partial_2 \mcq^m$ in a neighborhood of $\thv$ is the $C^2$-differentiability of $\Phi$ and $f$ in a neighborhood of $\thv$.
    Under all those assumptions, mirror prox thus meets \ref{hyp:main:regularity}-\ref{hyp:main:thv}.
    
    Note that the above sufficient condition of regularity implies \ref{hyp:main:regularity} for mirror descent, and that if $\tBv\succ 0$, then $\partial^2 \tilde{f}(\thv)\succ 0$ implies \ref{hyp:main:thv} for mirror descent (see \Cref{ex:main:mirror_descent} in page \pageref{ex:main:mirror_descent}).
    
    As a remark, \eqref{eq:main:comm:mp} yields that the convergence rates defined in (\ref{eq:main:rhosup}-\ref{eq:main:rhoinf}) are always strictly higher for mirror prox than for the corresponding mirror descent.
    The rate $\rhoinf^m$ is even lower-bounded by $3/4$ (see \Cref{sec:app:thm:main}).
\end{manualexample}

\begin{manualexample}{3}[Newton's method] \label{ex:main:newton}
    Let $f$ be a $C^2$-differentiable function $f$ whose Hessian is invertible on $\Theta$. Newton's method considers the procedure defined for all $n\in\nset$ by
    \begin{equation*}
        \theta_{n+1} = \theta_n - \partial^2 f(\theta_n)^{-1} \partial f(\theta_n).
    \end{equation*}
    It fits into the general framework of \eqref{eq:general_framework:m} with $\mcq$ defined on $\Theta\times\Theta$ by
    \begin{equation*}
        \Q{\theta}{\theta'} \eqdef \frac 1 2 \norm{\theta' - \theta + \partial^2 f(\theta)^{-1} \partial f(\theta)}_2^2.
    \end{equation*}
    If $f$ is thrice differentiable at $\thv$ and $\partial f(\thv) =0$, straightforward calculus yields $\Av = I_q$ and $\Bv=0$.
    Newton's method thus meets \ref{hyp:main:thv} with $\rhoinf=\rhosup=0$, which is coherent with the fact that the convergence is quadratic under the assumptions of \cite[Theorem 3.5, p.44]{nocedal_wright}.
\end{manualexample}

\section{Proof of Theorem 1} \label{sec:main:proof}
We start with some notation that will be used in several parts of the paper. 
Set $d \eqdef \dim(\Vset)$.
Let $v_1,\ldots,v_d \in\rset^q$ be an orthonormal basis of $\Vset$ and let $P$ be the matrix 
\begin{equation}
    P\eqdef [v_1 | \ldots | v_d] \in \rmat{q}{d} \label{eq:def:P}
\end{equation}
so that $\Vset=P (\rset^d)$.
For all $x \in \rset^q$ and $M \in \rmat{q}{q}$, write
\begin{equation}
  \label{eq:def:tilde}
\tilde{x} \eqdef P^\top x \in \rset^d\,, \quad  \tilde{M} \eqdef
P^\top M P \in \rmat{d}{d}.  
\end{equation}
Note that for all $v\in \Vset$, $P\tilde v = P P^\top v = v$.
Write for all $n \in \nset$, 
\begin{equation}
    \Delta_n \eqdef \theta_n-\thv \in \Vset, \label{eq:def:delta}
\end{equation}
and hence 
\begin{equation}
    \tilde \Delta_n \eqdef P^\top (\theta_n-\thv) \in \rset^d. \label{eq:def:tilde:delta}
\end{equation}
Then, for all $M \in \rmat{q}{q}$ and $n, m \in \nset$,
\begin{equation} \label{eq:notation:pscal}
    \tilde \Delta_n^\top \tilde M \tilde \Delta_m=\tilde \Delta_n^\top P^\top M P \tilde \Delta_m=\Delta_n^\top M \Delta_m.
  \end{equation}
In particular, with $M = I_q$ the identity matrix, for all $n \in \nset$,
\begin{equation} \label{eq:notation:norm}
    \normtxt{\tilde{\Delta}_n}_2 = \normtxt{\Delta_n}_2.
  \end{equation}

\begin{proof}[Proof of \Cref{thm:main}]

    In the definition of $\rhosup$ given in \eqref{eq:main:rhosup}, the supremum of $v\mapsto |v^\top \Bv v| / (v^\top \Av v)$ can be taken over the compact set $\settxt{v\in \Vset}{v^\top v = 1}$ and it is thus attained.
    This yields $\rhosup \in[0;1)$ under \ref{hyp:main:thv}.

    Let $\rho\in(\rhosup; 1)$. \Cref{prop:main:contraction} below provides for sufficiently large $n$,
    \begin{equation*}
        \normtxt{\tilde{\Delta}_{n+1}}_{\tAv} \leq \rho \normtxt{\tilde{\Delta}_n}_{\tAv}.
    \end{equation*}
    This yields $\normtxt{\tilde{\Delta}_n}_{\tAv} = \bigo(\rho^n)$, and hence $\normtxt{\tilde{\Delta}_n}_{2} = \bigo(\rho^n)$ by the equivalence of norms in finite dimension.
    The proof is concluded by noting that this holds for any arbitrary $\rho> \rhosup$ and since by \eqref{eq:notation:norm},
    \begin{equation*}
        \normtxt{\tilde{\Delta}_n}_2 = \normtxt{\Delta_n}_2 = \norm{\theta_n - \thv}_2.
    \end{equation*}
\end{proof}

\begin{lemma} \label{lemma:main:minimizer}
    Under \ref{hyp:main:convergence}-\ref{hyp:main:regularity}, $\thv$ is a local minimizer on $\Theta$ of the function $\theta\mapsto\Q{\thv}{\theta}$.
\end{lemma}
\begin{proof}
    Let $\Nset$ be a neighborhood of $\thv$ such that $\mcq$ is continuous on $\Nset\times\Nset$.
    For all $\theta \in \Nset$ and $n \in \nset$, the definition of $\seq{\theta_n}$ in \eqref{eq:general_framework:m} provides $\Qtxt{\theta_n}{\theta_{n+1}} \leq \Qtxt{\theta_n}{\theta}$.
    Taking the limit when $n$ goes to infinity yields $\Qtxt{\thv}{\thv} \leq \Qtxt{\thv}{\theta}$ for all $\theta\in\Nset$.
\end{proof}

\begin{proposition} \label{prop:main:contraction}
    Under \ref{hyp:main:convexity}-\ref{hyp:main:thv}, for all $\rho>\rhosup$, for sufficiently large $n$,
    \begin{equation*}
        \normtxt{\tilde{\Delta}_{n+1}}_{\tAv} \leq \rho \normtxt{\tilde{\Delta}_n}_{\tAv}.
    \end{equation*}
\end{proposition}
\begin{proof}
    By \Cref{lemma:main:minimizer}, $\thv$ is a local minimizer of the function $\theta \mapsto \Qtxt{\thv}{\theta}$.
    From the differentiability of that function at $\thv$ under \ref{hyp:main:regularity}, and the convexity of $\Theta$ under \ref{hyp:main:convexity}, we deduce that for all $\theta\in\Theta$,
    \begin{equation} \label{eq:main:grad_min:thv}
        \pscal{\donce{\Q{\thv}{\thv}}}{\theta - \thv} = \lim_{t \to 0^+}  \frac{\partial \Qtxt{\thv}{t \theta +(1-t)\thv}}{\partial t} \geq 0.
    \end{equation}
    Similarly, under \ref{hyp:main:convergence}-\ref{hyp:main:regularity} the function $\theta \mapsto \Qtxt{\theta_n}{\theta}$ is differentiable at $\theta_{n+1}$ for sufficiently large $n$, which yields for all $\theta \in \Theta$,
    \begin{equation} \label{eq:main:grad_min:theta_n}
        \pscal{\donce{\Q{\theta_n}{\theta_{n+1}}}}{\theta - \theta_{n+1}} \geq 0.
    \end{equation}
    Using \eqref{eq:main:grad_min:theta_n} with $\theta=\thv$ and \eqref{eq:main:grad_min:thv} with $\theta=\theta_{n+1}$  provides
    \begin{equation*}
        \pscal{\donce{\Q{\theta_n}{\theta_{n+1}}}}{\theta_{n+1} - \thv} \leq 0 \leq \pscal{\donce{\Q{\thv}{\thv}}}{\theta_{n+1} - \thv},
    \end{equation*}
    which in turn implies
    \begin{equation} \label{eq:main:balakrishnan}
        \pscal{\donce{\Q{\theta_n}{\theta_{n+1}}} - \donce{\Q{\theta_n}{\thv}}}{\theta_{n+1} - \thv} \leq
        \pscal{\donce{\Q{\thv}{\thv}} - \donce{\Q{\theta_n}{\thv}}}{\theta_{n+1} - \thv}.
    \end{equation}
    Besides, applying Taylor's theorem to $\theta \mapsto \donce{\Qtxt{\theta_n}{\theta}}$ and $\theta \mapsto \donce{\Qtxt{\theta}{\thv}}$ yields for sufficiently large $n$,
    \begin{align}
        \donce{\Q{\theta_n}{\theta_{n+1}}} - \donce{\Q{\theta_n}{\thv}} &= \An (\theta_{n+1} - \thv), \label{eq:main:d2_an}  \\ 
        \donce{\Q{\thv}{\thv}}-\donce{\Q{\theta_n}{\thv}}  &= \Bn (\theta_n - \thv), \label{eq:main:d2_bn}
    \end{align}
    where
    \begin{equation} \label{eq:notation:an_bn}
        \An \eqdef\int_0^1 \dtwice{2}{2}{\Q{\theta_n}{s \theta_{n+1}+(1-s) \thv}} \rmd s, \quad \Bn  \eqdef -\int_0^1 \dtwice{1}{2}{\Q{s \thv+(1-s) \theta_n}{\thv}} \rmd s.
    \end{equation}
    Plugging (\ref{eq:main:d2_an}-\ref{eq:main:d2_bn}) into $\eqref{eq:main:balakrishnan}$, we deduce
    \begin{equation*} 
        (\theta_{n+1} - \thv)^\top \An (\theta_{n+1} - \thv) \leq (\theta_{n+1} - \thv)^\top \Bn (\theta_n - \thv).
    \end{equation*}
    Using \eqref{eq:notation:pscal}, this can be written as
    \begin{equation} \label{eq:main:inequality_bala:tilde}
        \tilde{\Delta}_{n+1}^\top \tAn \tilde{\Delta}_{n+1} \leq \tilde{\Delta}_{n+1}^\top \tBn \tilde{\Delta}_{n}.  
    \end{equation}
    Now, by Schwarz's theorem, $\Av$ and hence $\tAv$ are
    symmetric. Similarly, under
    \ref{hyp:main:convergence}-\ref{hyp:main:regularity}, $\An$ and
    hence $\tAn$ are symmetric for sufficiently large $n$. Moreover, \ref{hyp:main:thv} implies the positive-definiteness of $\tAv$, and by \ref{hyp:main:convergence}-\ref{hyp:main:regularity} that of $\tAn$ for sufficiently large $n$ (see \cite[Section 1.3.4, p.47]{terry_tao}). We can thus apply \Cref{lemma:tech:domination} to \eqref{eq:main:inequality_bala:tilde} with $x=\tilde{\Delta}_{n+1}$, $y=\tilde{\Delta}_{n}$, $A=\tAn$ and $B=\tBn$ and we obtain
    \begin{equation} \label{eq:main:non_asymptotic}
        \normtxt{\tilde{\Delta}_{n+1}}_{\tAn} \leq \rhon \normtxt{\tilde{\Delta}_n}_{\tAn},
    \end{equation}
    where $\rhon \eqdef \normmattxt{\tAn^{-1/2} \tBn \tAn^{-1/2}}_2$. Under \ref{hyp:main:convergence}-\ref{hyp:main:regularity}, \Cref{lemma:tech:rho_approx} shows that $\rhon$ converges to $\normmattxt{\tAv^{-1/2} \tBv \tAv^
    {-1/2}}_2$ by choosing $A = \tAv$, $B = \tBv$, $M = \tAn-\tAv$ and $N = \tBn-\tBv$.  On the other hand, by \Cref{lemma:tech:rho_a_b}, $\rhosup = \normmattxt{\tAv^{-1/2} \tBv \tAv^
    {-1/2}}_2$. 
    
    Let $\rho > \rhosup$. Set $\rho' \eqdef (\rho+\rhosup)/2$ and $\epsilon > 0$ such that $(1+\epsilon) \rho' \leq (1-\epsilon) \rho$.
    Under \ref{hyp:main:convergence}-\ref{hyp:main:regularity}, \Cref{lemma:tech:norm_approx} yields that for sufficiently large $n$, for all $u \in \rset^d$,
    \begin{equation*}
        (1-\epsilon) \norm{u}_{\tAn} \leq \norm{u}_{\tAv} \leq (1+\epsilon) \norm{u}_{\tAn}.
    \end{equation*}
    Combining with \eqref{eq:main:non_asymptotic} and the convergence of $\rhon$ to $\rhosup$, we deduce for sufficiently large $n$,
    \begin{align*}
        \normtxt{\tilde{\Delta}_{n+1}}_{\tAv} \leq (1+\epsilon) \normtxt{\tilde{\Delta}_{n+1}}_{\tAn} \leq (1+\epsilon) \rhon \normtxt{\tilde{\Delta}_n}_{\tAn}
        \leq \frac{(1+\epsilon) \rho'}{1-\epsilon}  \normtxt{\tilde{\Delta}_n}_{\tAv} \leq \rho \normtxt{\tilde{\Delta}_n}_{\tAv}. 
    \end{align*}
\end{proof}

\section{Convex constrained optimization} \label{sec:mirror_prox}

\begin{theorem} \label{thm:convergence_mirror_prox}
    Assume that the mirror prox strategy defined in \Cref{ex:mirror_prox} page \pageref{ex:mirror_prox} satisfies the following assumptions:
    \hypertarget{hyp:conv:mp:i}{(i)} $\Cset\subset\Dset$,
    \hypertarget{hyp:conv:mp:ii}{(ii)} $\Phi$ and $f$ are twice differentiable on $\Cset$ and $C^2$-differentiable in a neighborhood of $\thv$,
    \hypertarget{hyp:conv:mp:iii}{(iii)} $\Phi$ is $\gamma$-strongly convex on $\Cset$ and $f$ is convex and $\beta$-smooth, with respect to $\norm{\cdot}_2$,
    \hypertarget{hyp:conv:mp:iv}{(iv)} $\eta \in (0;\gamma/\beta)$,
    \hypertarget{hyp:conv:mp:v}{(v)} $\thv$ is the unique minimizer of $f$ on $\Cset$,
    \hypertarget{hyp:conv:mp:vi}{(vi)} $\thv\in\ri{\Cset}$ and $\partial^2 \tilde{f}(\thv)\succ 0$.    

    Then, the algorithm converges and the convergence is asymptotically geometric.
\end{theorem}
\begin{proof}
    See \Cref{prop:suff:mp} in page \pageref{prop:suff:mp} and \Cref{ex:main:mirror_prox} in page \pageref{ex:main:mirror_prox}.
\end{proof}

Even if the convergence rates of mirror descent are always lower than those of mirror prox for the same optimization problem (see \Cref{ex:main:mirror_prox} in page \pageref{ex:main:mirror_prox}), the convergence of mirror prox is guaranteed under the assumptions of \Cref{thm:convergence_mirror_prox}.

Note that no conditions are imposed on the initialization (see \cite[Chapter 4, p.299]{bubeck}).

\begin{corollary} \label{corollary:convergence_mirror_prox}
    Let $q\in\nset^*$, $\Cset\subset\rset^q$ be compact set, and $f$ be a function that meets assumptions \hyperlink{hyp:conv:mp:ii}{(ii)}-\hyperlink{hyp:conv:mp:iii}{(iii)} and \hyperlink{hyp:conv:mp:v}{(v)}-\hyperlink{hyp:conv:mp:vi}{(vi)} of \Cref{thm:convergence_mirror_prox}.
    Then, mirror prox provides an algorithm that converges to $\argmin_\Cset f$.
\end{corollary}
\begin{proof}
    Write $R \eqdef \max_{x\in\Cset} \norm{x}_2$. For all $R'>R$, the mirror map $\Phi$ defined on $\Dset \eqdef \Ball{0}{R'}\eqdef\set{x\in\rset^d}{\norm{x}_2< R'}$ by $\Phi(x)\eqdef\norm{x}_2^2/(R'-\norm{x}_2^2)$ meets the assumptions of \Cref{thm:convergence_mirror_prox} for all $\eta \in (0;2(R'\beta)^{-1})$.
\end{proof}

\section{Discussion} \label{sec:discussion}

\subsection{Non-asymptotic convergence} \label{sec:discussion:nonasymptotic}

We proved the asymptotic geometric convergence in \Cref{sec:main:proof} by using that for all $n \in \nset$,
\begin{equation*}
    \normtxt{\ttheta_{n+1} - \ttheta_\star}_{\tAn} \leq \rhon \normtxt{\ttheta_n - \ttheta_\star}_{\tAn},
\end{equation*}
as soon as we can define the above quantities (see \eqref{eq:main:non_asymptotic}).
The question then rises of deriving non-asymptotic convergence rates. Apart from the fact that the norm depends on $n$, the main issue would be to obtain $\rhon < 1$.
However, the ratio $\rhosup$ compares $\dtwice{2}{2}{\Qtxt{\thv}{\thv}}$ with $\dtwice{1}{2}{\Qtxt{\thv}{\thv}}$, whereas $\rhon$ compares
\begin{equation} \label{eq:discussion:an_bn}
    \An=\int_0^1 \dtwice{2}{2}{\Q{\theta_n}{s \theta_{n+1}+(1-s) \thv}} \rmd s \quad \mbox{with} \quad \Bn=-\int_0^1 \dtwice{1}{2}{\Q{s \thv+(1-s) \theta_n}{\thv}} \rmd s,
\end{equation}
and the problem is not of the same complexity. A sufficient condition to simplify it can be $\min \Spec{\An} > \max |\Spec{\Bn}|$.
Despite being less precise than a comparison being matrices, such a condition has the advantage that it is sufficient to verify it for every $s \in [0;1]$ in \eqref{eq:discussion:an_bn}.
That pointwise condition essentially corresponds to the conditions 1 and 2 behind $\gamma < \lambda$ in \cite[Theorem 1]{balakrishnan}, concerning $\dtwice{1}{2}{\mcq}$ and $\dtwice{2}{2}{\mcq}$ respectively
(the proof of that theorem has besides inspired this work).
We can also identify the classical assumptions of smoothness and
Lipschitz continuity for gradient descent and its variants.

In light of that remark, we better understand why the framework introduced in this paper yields better results asymptotically, as it allows to work with the true asymptotic convergence rate. We can see it when comparing with the results stated so far in the EM literature \cite{dempster,kunstner,meng_94_r} (that generally do not consider constrained optimization and assume that the mapping $\mcm$ is differentiable, among other things).

\subsection{Quadratic convergence}

Under the assumptions of \Cref{thm:main:lower_bound}, we established in \Cref{sec:app:main} that for sufficiently large $n$, we can write $\tBn \tilde{\Delta}_n = \tAn \tilde{\Delta}_{n+1}$, which is equivalent to
\begin{equation} \label{eq:discussion:thv}
    \tthv = \ttheta_n + (I-\tAn^{-1}\tBn)^{-1} (\ttheta_{n+1}-\ttheta_n).
\end{equation}
Note that $\tAn$ and $\tBn$ cannot be computed as they depend on $\thv$ (see \eqref{eq:notation:an_bn}). However, we can approximate them using only $\theta_n$ in order to estimate iteratively $\thv$ with \eqref{eq:discussion:thv}.
In the example of unconstrained gradient descent with step-size $1$, using $\hat{\mca}_n = I_q$ and $\hat{\mcb}_n = I_q - \partial^2 f(\theta_n)$ (see \Cref{ex:main:mirror_descent} in page \pageref{ex:main:mirror_descent}) corresponds to Newton's method (see \Cref{ex:main:newton}).

\subsection{Non-convex constrained optimization}

Considering \Cref{corollary:convergence_mirror_prox} and Lemmas \ref{lemma:tech:sup_f_sup_m} and \ref{lemma:tech:biconjugate}, the problem of finding the unique minimizer of non-convex functions can be brought down to finding $\beta$-smooth approximations of their biconjugates (for an arbitrary $\beta\in\rset^*_+$).

\appendix

\section{Proofs} \label{sec:app:proofs}

\subsection{Asymptotic convergence rate} \label{sec:app:main}

\begin{proof}[Proof of \Cref{cor:main:Q}]
Let $\norm{\cdot}$ be any norm on $\rset^q$. Under \ref{hyp:main:regularity}, the Taylor expansion with integral remainder yields 
\begin{align}
    \Q{\theta_n}{\theta_{n+1}} - \Q{\theta_n}{\thv} &= \lr{\int_0^1 \partial_2 \Q{\theta_n}{t \theta_{n+1}+(1-t)\thv} \rmd t }(\theta_{n+1}-\thv) \nonumber\\
    &= \partial_2 \Q{\thv}{\thv}(\theta_{n+1}-\thv)+ o(\norm{\theta_{n+1}-\thv}) \label{eq:cor1:one}
\end{align}
where the last equality follows from the continuity of the function $(\theta,\theta') \mapsto \partial_2 \Q{\theta}{\theta'}$. 
Moreover, since $\theta \mapsto \Q{\theta}{\thv}$ is differentiable at $\thv$, 
\begin{equation}
    \Q{\theta_n}{\thv} - \Q{\thv}{\thv} = \partial_1 \Q{\thv}{\thv} (\theta_n -\thv) + o(\norm{\theta_{n}-\thv}) \label{eq:cor1:two}
\end{equation}
Summing \eqref{eq:cor1:one} and \eqref{eq:cor1:two} combined with \Cref{thm:main} yield the expected result.
\end{proof}

\begin{proof}[Proof of \Cref{thm:main:lower_bound}]
    First, note that the theorem is proved if $\rhoinf=0$. We now assume that $\rhoinf>0$, which in particular implies that $\tBv$ is invertible. In what follows, we use the notation introduced in \Cref{sec:main:proof}. 
    Following the proof of \Cref{thm:main}, see in particular the proof of Proposition~\ref{prop:main:contraction}, with the additional assumption that $\thv \in \ri{\Theta}$, we can prove that for sufficiently large $n$, for all $\theta \in \Theta$,
    \begin{equation*}
        \pscal{\donce{\Q{\theta_n}{\theta_{n+1}}}}{\theta - \theta_{n+1}} =
        \pscal{\donce{\Q{\thv}{\thv}}}{\theta - \thv} = 0.
    \end{equation*}
    In other words, $\donce{\Qtxt{\theta_n}{\theta_{n+1}}}, \donce{\Qtxt{\thv}{\thv}} \in \Vset^\perp$. This implies that for sufficiently large $n$,
    \begin{align}
        \donce{\Q{\thv}{\thv}} - \donce{\Q{\theta_n}{\theta_{n+1}}}
        &= \donce{\Q{\thv}{\thv}} - \donce{\Q{\theta_n}{\thv}} + \donce{\Q{\theta_n}{\thv}} - \donce{\Q{\theta_n}{\theta_{n+1}}} \nonumber \\
        &= \Bn (\theta_n - \thv) - \An (\theta_{n+1} - \thv) \nonumber \\
        &= \Bn P \tilde{\Delta}_n - \An P \tilde{\Delta}_{n+1} \in \Vset^\perp, \label{eq:thm-mino:one}
    \end{align}
    where $\An$,$\Bn$, $P$ and $\tilde \Delta_n$ are defined respectively in \eqref{eq:notation:an_bn}, \eqref{eq:def:P} and \eqref{eq:def:tilde:delta}. As by definition of $P$, the condition $v \in \Vset^\perp$ is equivalent to the identity $P^\top v = 0$, we deduce from \eqref{eq:thm-mino:one} that $ P^T \Bn P \tilde{\Delta}_n = P^T \An P \tilde{\Delta}_{n+1}$, which can be written as
\begin{equation}
\label{eq:bdelta_adelta}
\tBn \tilde{\Delta}_n = \tAn \tilde{\Delta}_{n+1}. 
\end{equation}
Moreover, by \ref{hyp:main:thv},  $\tAv$ is positive-definite and by \ref{hyp:main:convergence}-\ref{hyp:main:regularity},  $\tAn$ is also positive-definite for sufficiently large $n$. Then, the invertibility of $\tBv$  allows to write for sufficiently large $n$:
    \begin{equation*}
        \lr{\tAn^{-1/2} \tBn \tAn^{-1/2}}^{-1} \tAn^{1/2} \tilde{\Delta}_{n+1} =  \tAn^{1/2}\tBn ^{-1}\tAn \tilde{\Delta}_{n+1}=\tAn^{1/2}\tBn ^{-1}\tBn \tilde{\Delta}_{n}  = \tAn^{1/2} \tilde{\Delta}_n,
    \end{equation*}
    and thus, combining with $\normtxt{\tilde{\Delta}_n}_{\tAn}=\|\tAn^{1/2} \tilde{\Delta}_n\|_2 $ and $\normtxt{\tilde{\Delta}_{n+1}}_{\tAn}=\|\tAn^{1/2} \tilde{\Delta}_{n+1}\|_2 $, we get 
    \begin{equation} \label{eq:main:thm2:domination}
        \normtxt{\tilde{\Delta}_n}_{\tAn} \leq \normmat{\lr{\tAn^{-1/2} \tBn \tAn^{-1/2}}^{-1}}_2 \normtxt{\tilde{\Delta}_{n+1}}_{\tAn}.
    \end{equation}
    Besides, by the symmetry of $\tAv^{-1/2} \tBv \tAv^{-1/2}$,
    \begin{align} \label{eq:main:thm2:rho}
        \begin{split}
            \normmat{\lr{\tAv^{-1/2} \tBv \tAv^{-1/2}}^{-1}}_2^{-1}
            &= \rhom{\lr{\tAv^{-1/2} \tBv \tAv^{-1/2}}^{-1}}^{-1} \\
            &= \inf_{u\in\rset^d} \frac {|u^\top \tAv^{-1/2} \tBv \tAv^{-1/2} u|} {u^\top u}
            = \inf_{u\in\rset^d} \frac {|u^\top \tBv u|} {u^\top \tAv u} = \rhoinf.
        \end{split}
    \end{align}
    Let $\rho\in(0;\rhoinf)$.
    Following the same steps as for the proof of \Cref{thm:main}, we can prove that $\normmattxt{(\tAn^{-1/2} \tBn \tAn^{-1/2})^{-1}}_2$ converges to $\normmattxt{(\tAv^{-1/2} \tBv \tAv^{-1/2})^{-1}}_2$. 
    Together with \eqref{eq:main:thm2:domination} and \eqref{eq:main:thm2:rho}, this provides the existence of $n_0\in\nset$ such that for all $n\geq n_0$, $\normtxt{\tilde{\Delta}_n}_{\tAv} \leq \rho^{-1} \normtxt{\tilde{\Delta}_{n+1}}_{\tAv}$.
    We deduce by induction that for all $n\geq n_0$, $\normtxt{\tilde{\Delta}_{n_0}}_{\tAv} \leq \rho^{-(n-n_0)} \normtxt{\tilde{\Delta}_{n}}_{\tAv}$.
    As the sequence $\seq{\theta_n}$ is not eventually equal to $\thv$, by \eqref{eq:notation:norm} we can choose $n_0$ such that $\normtxt{\tilde{\Delta}_{n_0}}_{\tAv}\neq0$. Hence, since all the norms on a finite dimensional space are equivalent, we deduce $\underset{n\to\infty}{\lim \inf} \eqsp \frac 1 n \log\norm{\tilde \Delta_n}_{2}=\underset{n\to\infty}{\lim \inf} \eqsp \frac 1 n \log\norm{\tilde \Delta_n}_{\tAv} \geq \log \rho $. The proof is then concluded by applying \eqref{eq:notation:norm} and by noting that $\rho$ is arbitrary in $(0;\rhoinf)$. 
\end{proof}

\begin{proof}[Proof of \Cref{thm:main:exact_rate}]
   In this proof, we use the notation introduced in \Cref{sec:main:proof}.  Define 
   $$ 
   \tSv\eqdef \tAv^{-1}\tBv. 
   $$
   Note that $\tSv$ is similar to the symmetric matrix $\tAv^{-1/2}\tBv\tAv^{-1/2}$ and is therefore diagonalizable. Therefore, there exists an invertible matrix $R=[R(i,j)]_{1\leq i,j \leq d} \in \rset^{d\times d}$ such that $\tSv=\tAv^{-1}\tBv=R  \tDv R^{-1}$ where $\tDv$ is a diagonal matrix. For any matrix $S \in \rset^{d\times d}$, it is convenient to use the notation $S^R=R^{-1}S R$. In particular, we have $\tSv^R=\tDv$. Moreover, for any vector $\Delta \in \rset^d$, we use the notation $\Delta^R\eqdef R^{-1} \Delta$. Write
    \begin{equation*}
        \tSn \eqdef \tAn^{-1}\tBn,
    \end{equation*}
    where $\An$ and $\Bn$ are defined in \eqref{eq:notation:an_bn} and $\tSn$ is well-defined for sufficiently large $n$ as $\tAv$ is positive-definite by \ref{hyp:main:thv}, and using \ref{hyp:main:convergence}-\ref{hyp:main:regularity}. Besides, Theorems \ref{thm:main} and \ref{thm:main:lower_bound} provide $\rhoinf, \rhosup\in[0;1)$. As the theorem is proved if $\rhoinf=0$, we now assume that $\rhoinf>0$, which implies the invertibility of $\tBv$  and hence of $\tSv$.
    Moreover, \Cref{thm:main} also yields that for all $\rho\in(\rhosup;1)$, $\theta_n-\thv=\lito(\rho^n)$.
    We deduce by the $C^2$-differentiability of $\partial_2 \mcq$ in a neighborhood of $(\thv, \thv)$ that for all $\rho\in(\rhosup;1)$,
    \begin{equation} \label{eq:thm_exact:dn_dv}
        \lr{\tSn^R}^{-1} - \tDv^{-1}=\lr{\tSn^R}^{-1} -\lr{\tSv^R}^{-1} = \lito(\rho^n).
    \end{equation}
    Following the proof of \Cref{thm:main:lower_bound},  by \eqref{eq:bdelta_adelta} there exists $n_0\in\nset$ such that for all $n\geq n_0$, $\tBn \tilde{\Delta}_n = \tAn \tilde{\Delta}_{n+1}$, that is, $\tSn \tilde{\Delta}_n = \tilde{\Delta}_{n+1}$ or equivalently $\tSn^R \tilde{\Delta}_n^R = \tilde{\Delta}_{n+1}^R$. This implies for all $m\in\nset^*$, 
    \begin{align*}
        \tilde{\Delta}_n^R &= \lr{\tSn^R}^{-1} \tilde{\Delta}_{n+1}^R = \lrb{\prod_{k=0}^{m-1} \lr{\tS_{n+k}^R}^{-1}} \tilde{\Delta}_{n+m}^R, \\
        &= \tDv^{-m}  \tilde{\Delta}_{n+m}^R + \sum_{k=0}^{m-1} \lrc{\tDv^{-m+1+k} \lrb{\lr{\tS_{n+k}^R}^{-1}-\tDv^{-1}} \prod_{l=0}^{k-1} \lr{\tS_{n+l}^R}^{-1}} \tilde{\Delta}_{n+m}^R.
    \end{align*}
    Component-wise, this yields for such $n,m\in\nset^*$ that for all $i\in\integerval{d}$,
    \begin{equation*}
        \tilde{\Delta}_n^R (i) = \tDv(i,i)^{-m} \tilde{\Delta}_{n+m}^R (i) + \sum_{k=0}^{m-1} \tDv(i,i)^{-m+1+k} L_{n,m,k}(i)^\top \tilde{\Delta}_{n+m}^R,
    \end{equation*}
    where  for any $k\in\integerval[0]{m-1}$, $L_{n,m,k}(i)^\top$ denotes the $i$-th row of the matrix 
    $$ 
    \lrb{\lr{\tS_{n+k}^R}^{-1}-\tDv^{-1}} \prod_{l=0}^{k-1} \lr{\tS_{n+l}^R}^{-1}.   
    $$
    Recalling that $\normmat{\cdot}_F$ is the Frobenius norm, we let $C_F>0$ be constant such that $\normmat{\cdot}_F \leq \normmat{\cdot}_2 C_F$ on $\rmat{d}{d}$.
    Let $i\in\integerval{d}$ such that $|\tDv(i,i)|=\rhosup$.
    Using the Cauchy-Schwarz inequality we deduce 
    \begin{align}
        |\tilde{\Delta}_n^R (i)| &\leq \rhosup^{-m} \normtxt{\tilde{\Delta}_{n+m}^R}_2 + \sum_{k=0}^{m-1} \rhosup^{-m+1+k} \normtxt{L_{n,m,k}(i)}_2 \normtxt{\tilde{\Delta}_{n+m}^R}_2, \nonumber \\
        &\leq \rhosup^{-m} \normtxt{\tilde{\Delta}_{n+m}^R}_2 + \rhosup^{-m} \sum_{k=0}^{m-1} \rhosup^{k+1} \normmat{\lrb{\lr{\tS_{n+k+1}^R}^{-1}-\tDv^{-1}} \prod_{l=0}^{k-1} \lr{\tS_{n+l}^R}^{-1}}_F \normtxt{\tilde{\Delta}_{n+m}^R}_2, \nonumber \\
        &\leq \rhosup^{-m} \lrb{1 + C_F \sum_{k=0}^{m-1} \rhosup^{k+1} \normmattxt{\lr{\tS_{n+k+1}^R}^{-1}-\tDv^{-1}}_2 \prod_{l=0}^{k-1} \normmattxt{\lr{\tS_{n+l}^R}^{-1}}_2} \normtxt{\tilde{\Delta}_{n+m}^R}_2. \label{eq:thm:_exact:randal_trick}
    \end{align}
Let $\delta>\max(\rhosup^{-1},\rhosup \rhoinf^{-1})$. Pick $\rho \in (\rhosup,1)$ and $\epsilon>0$ such that $\rho (\rhoinf^{-1}+\epsilon)<\delta$. By \eqref{eq:thm_exact:dn_dv} there exists $C>0$ and $n_1\geq n_0$ such that for all $n\geq n_1$,
\begin{equation} \label{eq:thm_exact:gap_d}
    \normmattxt{ \lr{\tSn^R}^{-1} - \tDv^{-1}}_2 \leq C \rho^n \leq \epsilon. 
\end{equation}
Then, \eqref{eq:thm:_exact:randal_trick} yields for all $n\geq n_1$ and $m\in\nset^*$, using $\rhosup^{-1}<\delta$ and $\rho (\rhoinf^{-1}+\epsilon) < \delta$, 
\begin{align}
    |\tilde{\Delta}_n^R (i)| &\leq \rhosup^{-m} \lrb{1 + CC_F \sum_{k=0}^{m-1} \rhosup^{k+1} \rho^{n+k+1} (\rhoinf^{-1}+\epsilon)^k} \normtxt{\tilde{\Delta}_{n+m}^R}_2 \nonumber\\
    &=  \lrc{\rhosup^{-m} + CC_F \rho^{n+1}\sum_{k=0}^{m-1} \rhosup^{k+1-m}  \lrb{\rho(\rhoinf^{-1}+\epsilon)}^k} \normtxt{\tilde{\Delta}_{n+m}^R}_2 \nonumber\\
    & \leq \lr{\delta^m + CC_F \rho^{n+1} \sum_{k=0}^{m-1} \delta^{m-1-k} \delta^k} \normtxt{\tilde{\Delta}_{n+m}^R}_2 \nonumber \\
    & =\lr{\delta^m + CC_F \rho^{n+1} \delta^{m-1} m} \normtxt{\tilde{\Delta}_{n+m}^R}_2. \label{eq:thm_exact:one}
\end{align}
We now show that there exists $n\geq n_1$ such that $\tilde{\Delta}_n^R (i) = R^{-1} P^T \Delta_n(i)\neq 0$. Indeed, otherwise, by \Cref{lem:decomp:base} below, there exists a basis $w_1,\ldots,w_d$ of $\Vset$ such that $\Delta_n=\theta_{n}-\thv \in \mathrm{Span}(w_j, j \in \integerval{d} \setminus \{i\})$ for any $n \geq n_1$ which contradicts the assumption $\mathrm{Span}(\Delta_n, n\geq n_1)=\Vset$ in the statement of \Cref{thm:main:exact_rate}. Therefore, we can choose $n\geq n_1$ such that the lhs of \eqref{eq:thm_exact:one} is strictly positive. This $n$ being chosen, take the log in the previous inequality and divide by $m$. Letting $m$ goes to infinity, we then obtain for any $\delta>\max(\rhosup^{-1},\rhosup \rhoinf^{-1})$, 
\begin{align*}
    \log (\delta^{-1}) &\leq \liminf_{m \to \infty} \frac 1 m \log{ \normtxt{\tilde{\Delta}_{m}^R}_2}\leq \liminf_{m \to \infty} \frac 1 m \log{ \normmattxt{R^{-1}}_2 \times \normtxt{\tilde{\Delta}_{m}}_2}\\
    &= \liminf_{m \to \infty} \frac 1 m \log{\normtxt{\tilde{\Delta}_{m}}_2}=\liminf_{m \to \infty} \frac 1 m \log{ \normtxt{\Delta_{m}}_2},     
\end{align*}
where the last equality follows from \eqref{eq:notation:norm}. The proof is completed since $\delta$ is arbitrary provided that $\delta>\max(\rhosup^{-1},\rhosup \rhoinf^{-1})$, which is equivalent to $\delta^{-1} <\min(\rhosup,\rhosup^{-1} \rhoinf)$. 
\end{proof}

\begin{lemma}
    \label{lem:decomp:base}
Let $\Vset$ be a $d$-dimensional linear subspace of $\rset^q$. Let  $v_1,\ldots,v_d \in\rset^q$ be an orthonormal basis of $\Vset$ and let $w_1,\ldots,w_d \in\rset^q$ be another basis obtained from $(v_i)_{1\leq i \leq d}$ by the change-of-basis matrix $R$, that is, for any $j \in \integerval{d}$, $w_j=\sum_{i=1}^d R(i,j) v_i$. 

Then, for any $\Delta \in \Vset$, the $i$-th component of the decomposition of $\Delta$ on the basis $(w_i)_{1\leq i \leq d}$ is $R^{-1} P^T \Delta(i)$, where $P$ is the matrix 
\begin{equation*}
    P\eqdef [v_1 | \ldots | v_d] \in \rmat{q}{d}. 
\end{equation*}
\end{lemma}
\begin{proof}
Decomposing the vector $\Delta \in \Vset$ on the basis $(v_j)_{1\leq j \leq d}$ and using  $v_j=\sum_{i=1}^d R^{-1}(i,j) w_i$, we get
$$ 
 \Delta=\sum_{j=1}^d (v_j^T \Delta) v_j= \sum_{i=1}^d \lrb{\sum_{j=1}^d R^{-1}(i,j)(v_j^T \Delta) } w_i = \sum_{i=1}^d [R^{-1}P^T \Delta(i)]w_i
$$

\end{proof}
\subsection{Comments on \ref{hyp:main:convexity}} \label{sec:app:sm}

\begin{proof}[Proof of \Cref{thm:sm}]
    The proof amounts to building an analogous framework that satisfies \ref{hyp:main:convexity}-\ref{hyp:main:thv}. The first step is to define a suitable minimization set that is convex.
    Write $d$ the dimension of the submanifold $\Sset$ and for all $x\in\rset^q$, $R>0$, we set $\Ball{x}{R} \eqdef \set{y\in\rset^q}{\norm{x-y}_2<R}$.
    
    Under \ref{hyp:sm:theta} there exist $\Uset_1$, $\Uset_2$ two open
    neighborhoods of $\thv$ and the null-vector $\0$ in $\rset^q$, respectively, and a
    $C^2$-diffeomorphism $\psi\colon\Uset_1\rightarrow\Uset_2$ such
    that $\psi(\thv)=\0$ and
    $\psi(\Uset_1\cap\Sset)=\Uset_2\cap(\rset^d\times\{0\}^{q-d})$.
    Let $\Nset$ be an open neighborhood of $\thv$ such that $\mcq$ meets the conditions of \ref{hyp:main:regularity} on $\Nset\times\Nset$. Define $\Uset\eqdef\Uset_1\cap\Nset\cap\mathring{\Eset}$, which is an open set containing $\thv$ by \ref{hyp:sm:theta}.
    Set $r>0$ such that $\Ball{\thv}{r}\subset\Uset$, and $\epsilon>0$ such that $\Ball{\0}{\epsilon}\subset\psi(\Ball{\thv}{r/2})$. Define $\Wset\eqdef \rset^d\times\{0\}^{q-d}$ and the convex set
    \begin{equation} \label{eq:app:sm:xi}
        \Xi \eqdef \Ball{0}{\epsilon}\cap\Wset.
    \end{equation}
    Write $\phi$ the corestriction of $\psi$ to $\Xi$, that is, $\phi\colon\psi^{-1}(\Xi)\rightarrow\Xi$ such that $\phi(x)=\psi(x)$ for any $x\in \psi^{-1}(\Xi)$. Note that $\phi$ is still a $C^2$-diffeomorphism. Define then
    \begin{align} \label{eq:app:sm:r}
        \begin{split}
            \mcr \colon &\Xi \times \Xi \longrightarrow \rset\\
            & (\zeta, \zeta') \mapsto \mcr_\zeta(\zeta')\eqdef \Q{\phi^{-1}(\zeta)}{\phi^{-1}(\zeta')}.
        \end{split}
    \end{align}
    Under \ref{hyp:main:convergence} there exists $n_0\in\nset$ such that $\theta_n\in\phi^{-1}(\Xi)$ for all $n\geq n_0$, which allows to define on $\Xi$ the sequence
    \begin{equation} \label{eq:app:sm:zeta}
        \seq{\zeta_n} \eqdef \seq{\phi(\theta_{n+n_0})}.
    \end{equation}
    We deduce from the definition of $\seq{\theta_n}$ in \eqref{eq:general_framework:m} and from $\phi^{-1}(\Xi)\subset\Uset\subset\Eset$ that for all $n\in\nset$ and $\zeta\in\Xi$,
    \begin{equation} \label{eq:app:sm:mini}
        \R{\zeta_n}{\zeta_{n+1}} = \Q{\theta_{n+n_0}}{\theta_{n+n_0+1}} \leq \Q{\theta_{n+n_0}}{\phi^{-1}(\zeta)} = \R{\zeta_n}{\zeta}.
    \end{equation}
    The framework defined by (\ref{eq:app:sm:xi}-\ref{eq:app:sm:zeta}) thus fits into \eqref{eq:general_framework:m} and meets \ref{hyp:main:convexity}-\ref{hyp:main:regularity} with $(\theta_n,\thv,\mcq)$ replaced by $(\zeta_n,\0,\mcr)$. For consistency of notation, we write $\zv\eqdef\phi(\thv)=\0$. We now prove that \ref{hyp:main:thv} is satisfied with $(\thv,\mcq)$ replaced by $(\zeta_\star,\mcr)$.
    
    Denote by $J_{\phi}$ the Jacobian matrix of $\phi$. The fact that $\phi^{-1}(\Xi)\subset\Uset\subset\Nset$ allows to write for all $\theta, \theta'\in\phi^{-1}(\Xi)$,    
    \begin{equation*}
        \partial_2 \Q{\theta}{\theta'} = J_{\phi}(\theta')^\top \partial_2 \R{\phi(\theta)}{\phi(\theta')} = \sum_{i=1}^d \lrb{\partial_2 \R{\phi(\theta)}{\phi(\theta')}}_i \cdot \partial \phi_i(\theta'),
    \end{equation*}
    using that the image of $\phi$ is included in $\rset^d\times\{0\}^{q-d}$, i.e. $\phi_i \equiv 0$ for all $i\in\integerval[d+1]{q}{}$. This yields
    \begin{align}
        \dtwice{1}{2}{\Q{\thv}{\thv}} &= J_{\phi}(\thv)^\top \times \dtwice{1}{2}{\R{\zv}{\zv}} \times J_{\phi}(\thv), \label{eq:app:sm:d12} \\
        \dtwice{2}{2}{\Q{\thv}{\thv}} &= J_{\phi}(\thv)^\top \times \dtwice{2}{2}{\R{\zv}{\zv}} \times J_{\phi}(\thv) + \sum_{i=1}^d \lrb{\partial_2 \R{\zv}{\zv}}_i \cdot \partial^2 \phi_i(\thv). \label{eq:app:sm:d22}
    \end{align}
    Besides, \eqref{eq:app:sm:mini} and \Cref{lemma:main:minimizer} provide that $\zv$ is a local minimizer of the function $\zeta\mapsto\R{\zv}{\zeta}$. Together with the convexity of $\Xi$, the fact that $\zv\in\ri{\Xi}$ and that $\Aff{\Xi}=\Wset$, this implies $\partial_2 \R{\zv}{\zv}\in\Wset^\perp$.
    As $\Wset = \rset^d\times\{0\}^{q-d}$, the second term of the rhs in \eqref{eq:app:sm:d22} is thus null.
    Combining with $\Tthv = J_{\phi}(\thv)^{-1}\Wset$ by definition of the tangent space of a submanifold, we deduce from (\ref{eq:app:sm:d12}-\ref{eq:app:sm:d22}) that the rates $\rhoinf, \rhosup$ defined in (\ref{eq:main:rhosup}-\ref{eq:main:rhoinf}) for $\mcr$ are equal to the rates $\rhoinfsm, \rhosupsm$ defined in \eqref{eq:sm:rhos} for $\mcq$. Satisfying \ref{hyp:main:thv} for $\mcr$ is thus equivalent to satisfying \ref{hyp:sm:thv} for $\mcq$.  

    Therefore, we can apply Theorems \ref{thm:main} and \ref{thm:main:lower_bound} to the sequence $\seq{\zeta_n}$ and we get for any $(\rho_1,\rho_2) \in (\rhosupsm,1) \times (0,\rhoinfsm)$
\begin{equation} \label{eq:nonconvex:one}
    \zeta_n-\zv = \lito(\rho_1^n) \quad \mbox{and} \quad \rho_2^n= \lito(\norm{\zeta_n-\zv }_2).  
\end{equation}
To relate with the speed of convergence of $\theta_n -\thv$, note that for all $n\in\nset$,
    \begin{align}
        \zeta_n-\zv &= \phi(\theta_{n+n_0}) - \phi(\thv) = \psi(\theta_{n+n_0}) - \psi(\thv), \nonumber \\
        \theta_{n+n_0}-\thv &= \phi^{-1}(\zeta_n) - \phi^{-1}(\zv), \label{eq:nonconvex:two}
    \end{align}
    and that $\psi^{-1}(\Xi)\subset\Ball{\thv}{r/2}$ by \eqref{eq:app:sm:xi}. The $C^1$-differentiability of $\psi$ on $\bar{\mathbf{B}}(\thv, r/2)$ and of $\psi^{-1}$ on $\overline{\Xi}$ provides the existence of $C>0$ such that $\sup_{\theta\in\Ball{\thv}{r/2}} \normmattxt{\rmd_{\psi}(\theta)}_2 \leq C$ and $\sup_{\zeta\in\Xi} \normmattxt{\rmd_{\phi^{-1}}(\zeta)}_2 \leq C$, where $\rmd_{\psi}$ and $\rmd_{\phi^{-1}}$ denote the differentials of $\psi$ and $\phi^{-1}$, respectively. Thus $\psi$ is Lipshitz on $\Ball{\thv}{r/2}$ and $\phi^{-1}$ on $\Xi$. Combining with \eqref{eq:nonconvex:two} and \eqref{eq:nonconvex:one} concludes the proof.
\end{proof}

\subsection{Comments on \ref{hyp:main:convergence}} \label{sec:app:comments}

We first prove a general version of \Cref{prop:main:contraction}.

\begin{proposition} \label{prop:app:contraction}
    Assume that \ref{hyp:main:convexity}, \ref{hyp:suff:m_thv}, \ref{hyp:main:regularity} and \ref{hyp:main:thv} hold.
    Let $\norm{\cdot}$ be a norm on $\rset^q$. Then, for all $\rho>\rhosup$, there exists $\delta > 0$ such that for all $\theta\in\Theta$, $\theta'\in\M{\theta}$,
    \begin{equation*}
        \norm{\theta-\thv} \vee \norm{\theta'-\thv}\leq\delta \implies \normtxt{\ttheta'-\tthv}_{\tAv} \leq \rho \normtxt{\ttheta-\tthv}_{\tAv},
    \end{equation*}
where the notation $\ttheta$, $\ttheta'$ and $\tthv$ are defined in \eqref{eq:def:tilde}.
\end{proposition}

\begin{proof}
    In this proof, we use the notation introduced in \Cref{sec:main:proof}. For all $\delta>0$, write $\Ball{\thv}{\delta} \eqdef \settxt{\theta\in\rset^q}{\norm{\theta-\thv}<\delta}$.
    Under \ref{hyp:main:regularity} there exists $\delta_0 > 0$ such that $\partial_2 \mcq$ is well-defined and $C^1$-differentiable on $\Ball{\thv}{\delta_0} \times \Ball{\thv}{\delta_0}$.
    By \ref{hyp:main:convexity}, \ref{hyp:suff:m_thv} and \ref{hyp:main:regularity}, we can prove, similarly to \eqref{eq:main:balakrishnan} in the proof of \Cref{prop:main:contraction}, that for all $\theta, \theta'\in \Ball{\thv}{\delta_0}$ with $\theta'\in\M{\theta}$,
\begin{equation*} 
        \pscal{\donce{\Q{\theta}{\theta'}} - \donce{\Q{\theta}{\thv}}}{\theta' - \thv} \leq
        \pscal{\donce{\Q{\thv}{\thv}} - \donce{\Q{\theta}{\thv}}}{\theta' - \thv},
    \end{equation*}
which yields
    \begin{equation} \label{eq:suff:inequality_bala}
        (\theta' - \thv)^\top \At (\theta' - \thv) \leq (\theta' - \thv)^\top \Bt (\theta - \thv),
    \end{equation}
    where
    \begin{equation*}
        \At \eqdef\int_0^1 \dtwice{2}{2}{\Q{\theta}{s\theta'+(1-s)\thv}} \rmd s, \quad \Bt \eqdef -\int_0^1 \dtwice{1}{2}{\Q{s\thv+(1-s)\theta}{\thv}} \rmd s.
    \end{equation*}
    Under \ref{hyp:main:regularity}-\ref{hyp:main:thv} there also exists $\delta_1\in(0;\delta_0)$ such that if $\theta, \theta' \in \Ball{\thv}{\delta_1}$, then the symmetric matrix $\tAt$ is positive-definite.
    Using \eqref{eq:notation:pscal} and applying \Cref{lemma:tech:domination} to \eqref{eq:suff:inequality_bala} then provides
    \begin{equation} \label{eq:suff:domi_theta}
        \normtxt{\ttheta'-\tthv}_{\tAt} \leq \rhotheta \normtxt{\ttheta-\tthv}_{\tAt},
    \end{equation}
    where $\rhotheta \eqdef \normmattxt{\tAt^{-1/2} \tBt \tAt^{-1/2}}_2$.
    
    Let $\rho > \rhosup$. Set $\rho' \eqdef (\rho+\rhosup)/2$ and $\epsilon > 0$ such that $(1+\epsilon) \rho' \leq (1-\epsilon) \rho$.
    By \ref{hyp:main:regularity}-\ref{hyp:main:thv} and \Cref{lemma:tech:norm_approx} there exists $\delta_2 \in (0;\delta_1)$ such that for all $\theta, \theta' \in \Ball{\thv}{\delta_2}$, for all $u\in\rset^d$,
    \begin{equation*}
        (1-\epsilon) \norm{u}_{\tAt} \leq \norm{u}_{\tAv} \leq (1+\epsilon) \norm{u}_{\tAt}.
    \end{equation*}
    Moreover, by Lemmas~\ref{lemma:tech:rho_approx} and \ref{lemma:tech:rho_a_b}, under \ref{hyp:main:regularity} there exists $\delta_3 \in (0;\delta_2)$ such that $\rhotheta \leq \rho'$ for all $\theta, \theta' \in \Ball{\thv}{\delta_3}$.
    Combining with \eqref{eq:suff:domi_theta} yields that for all $\theta, \theta'\in \Ball{\thv}{\delta_3}$ with $\theta'\in\M{\theta}$,
    \begin{align*}
        \normtxt{\ttheta'-\tthv}_{\tAv} \leq (1+\epsilon) \normtxt{\ttheta'-\tthv}_{\tAt} &\leq (1+\epsilon) \rhotheta \normtxt{\ttheta-\tthv}_{\tAt} \\
        &\leq \frac{(1+\epsilon) \rho'}{1-\epsilon}  \normtxt{\ttheta-\tthv}_{\tAt} \leq \rho \normtxt{\ttheta-\tthv}_{\tAv}. 
    \end{align*}
\end{proof}

\begin{proof}[Proof of \Cref{thm:suff:convergence}]
    Applying \Cref{prop:app:contraction} to $\rho \eqdef (1+\rhosup)/2$ and the norm $\norm{\cdot}_2$ provides the existence of $\delta_0 > 0$ such that for all $\theta\in\Theta$, $\theta'\in\M{\theta}$,
    \begin{equation} \label{eq:suff:conv:domination}
        \normtxt{\theta-\thv}_2 \vee \normtxt{\theta'-\thv}_2 \leq \delta_0 \implies \normtxt{\ttheta'-\tthv}_{\tAv} \leq \rho \normtxt{\ttheta-\tthv}_{\tAv}.
    \end{equation}
    Moreover, by \Cref{lemma:tech:continuity_mapping_set} under \ref{hyp:suff:compacity}-\ref{hyp:suff:continuity} and by \ref{hyp:suff:m_thv}, there exists $\delta_1\in(0;\delta_0)$ such that for all $\theta\in\Theta$, $\theta'\in\M{\theta}$,
    \begin{equation*}
        \normtxt{\theta-\thv}_2 \leq \delta_1 \implies \normtxt{\theta'-\thv}_2 \leq \delta_0.
    \end{equation*}
    By \eqref{eq:notation:norm} and the equivalence of norms in finite dimension, combining with \eqref{eq:suff:conv:domination} yields the existence of $\delta>0$ such that for all $\theta\in\Theta$, $\theta'\in\M{\theta}$,
    \begin{equation*}
        \normtxt{\ttheta-\tthv}_{\tAv} \leq \delta \implies \normtxt{\ttheta'-\tthv}_{\tAv} \leq \rho \normtxt{\ttheta-\tthv}_{\tAv}.
    \end{equation*}
    Besides, by \ref{hyp:suff:accu_point} there exists $n_0 \in \nset$ such that $\normtxt{\ttheta_{n_0} - \tthv}_{\tAv} \leq \delta$.
    Using that $\rho < 1$ by \ref{hyp:main:thv}, we deduce by induction that for all $n \geq n_0$, $\normtxt{\ttheta_n - \tthv}_{\tAv} \leq \delta$, and that
    \begin{equation*}
        \normtxt{\ttheta_{n+1}-\tthv}_{\tAv} \leq \rho \normtxt{\ttheta_n-\tthv}_{\tAv},
    \end{equation*}
    which concludes the proof.
\end{proof}

\begin{lemma} \label{lemma:suff:m_equality}
    Assume that $\partial^2 \mcq$ is well-defined and differentiable on $\Theta$, and that for all $\theta, \theta'\in\Theta$, for all $v\in\Vset$, $v^\top \dtwice{1}{2}{\Q{\theta}{\theta'}}v < 0$.
    Then, for all $\theta, \theta' \in\Theta$,
    \begin{equation*}
        \M{\theta}\cap\M{\theta'}\cap\ri{\Theta}\neq\emptyset \implies \theta=\theta'.
    \end{equation*}
\end{lemma}
\begin{proof}
    Let $\theta''\in\M{\theta}\cap\M{\theta'}\cap\ri{\Theta}$. We can prove as for \Cref{thm:main:lower_bound} that it implies $\tBt(\ttheta-\ttheta') = \tAt (\ttheta'' - \ttheta'')=0$, where
    \begin{equation*}
        \tAt=\dtwice{2}{2}{\tQ{\theta}{\theta''}} \quad \mbox{and} \quad \tBt=-\int_0^1 \dtwice{1}{2}{\tQ{s \theta'+(1-s) \theta}{\theta''}} \rmd s.
    \end{equation*}
    Besides, by assumption, for all $v\in\Vset$, $v^\top \Bt v = -\int_0^1 v^\top \dtwice{1}{2}{\Q{s \theta'+(1-s) \theta}{\theta''}} v \rmd s > 0$.
    This provides the invertibilty of $\tBt$ and thus $\ttheta-\ttheta' = 0$, which concludes the proof by \eqref{eq:notation:norm}.
\end{proof}

\begin{proof}[Proof of \Cref{prop:suff:mp}]
    It is proved in \cite[Theorem 4.4, p.305]{bubeck} that under assumptions \hyperlink{hyp:ex:mp:cesaro}{(ii)} and \hyperlink{hyp:ex:mp:eta}{(iii)}, for all $\theta\in\Cset\cap\Dset$ and $n\geq 1$,
    \begin{equation*}
        \frac 1 n \sum_{i=0}^{n-1} f(\zeta_i) - f(\theta) \leq \frac 1 n \frac 1 {\eta} D_{\Phi}(\theta, \zeta_0).
    \end{equation*}
    We deduce by applying \Cref{lemma:tech:min_f_compact} under \hyperlink{hyp:ex:mp:mini}{(iv)} and the compacity of $\Cset$ that there exists $\varphi \colon \nset \rightarrow \nset$ strictly increasing such that $\seq{\zeta_{\varphi(n)}}$ converges to $\thv$.
    By \eqref{eq:ex:mp:zeta} this is equivalent to $\seq[n\in\nset^*]{\Mtxt{\theta_{\varphi(n)-1}}}$ converging to $\thv$.
    Besides, \hyperlink{hyp:ex:mp:mini}{(iv)} and the differentiability of $f$ provide $\M{\thv}=\{\thv\}$ (see \Cref{ex:suff:md} in page \pageref{ex:suff:md}), and by \Cref{lemma:tech:continuity_mapping} under \ref{hyp:suff:compacity}-\ref{hyp:suff:continuity} the function $\mcm$ is continuous on $\Theta$.
    We deduce that all accumulation points $\ell$ of the sequence $\seq[n\in\nset^*]{\theta_{\varphi(n)-1}}$ verify $\Mtxt{\ell} = \thv = \M{\thv}$.
    By \Cref{lemma:suff:m_equality} under \hyperlink{hyp:ex:mp:diff_2}{(i)}, \hyperlink{hyp:ex:mp:cesaro}{(ii)}, \hyperlink{hyp:ex:mp:eta}{(iii)} and \hyperlink{hyp:ex:mp:ri}{(v)} (using (\ref{eq:ex:main:md:q22}-\ref{eq:ex:main:md:q12})), this yields $\ell=\thv$ for all accumulation points, and thus the convergence of $\seq[n\in\nset^*]{\theta_{\varphi(n)-1}}$ to $\thv$ by the compacity of $\Cset$.
    
    Using that $\M{\thv}=\{\thv\}$, we can prove as in \Cref{ex:suff:md}, page \pageref{ex:suff:md}, that $\Mtxt[m]{\thv}=\{\thv\}$, where $\mcm^m$ is the minimization mapping corresponding to mirror prox (see \eqref{eq:ex:mp_q}).
\end{proof}

\begin{proof}[Proof of \Cref{thm:suff:m_thv}]
    Under \ref{hyp:suff:accu_point} there exists $\psi \colon \nset \rightarrow \nset$ strictly increasing such that $\seq{\theta_{\psi(n)}}$ converges to $\thv$.
    By the compacity of $\Theta \times \Theta$ under \ref{hyp:suff:compacity} there also exist $\thvv \in \Theta$ and $\varphi \colon \nset \rightarrow \nset$ strictly increasing such that
    \begin{equation} \label{eq:suff:subseq_couple}
        (\theta_{\varphi(n)}, \theta_{\varphi(n)+1}) \underset{n\to\infty}{\longrightarrow} (\thv, \thvv).
    \end{equation}
    By the monotonicity of the sequence $\seq{\vartheta(\theta_n)}$ under \ref{hyp:suff:lyapunov}, for all $n \in \nset$, $\vartheta(\theta_{\varphi(n+1)}) \leq \vartheta(\theta_{\varphi(n)+1}) \leq \vartheta(\theta_{\varphi(n)})$.
    Together with \eqref{eq:suff:subseq_couple} and the continuity of $\vartheta$ this yields
    \begin{equation*}
        \vartheta(\thvv) = \vartheta(\thv).
    \end{equation*}
    Besides, by the definition of $\seq{\theta_n}$, for all $\theta\in\Theta$,
$$
\mcq_{\theta_{\varphi(n)}}(\theta_{\varphi(n) + 1}) \leq \mcq_{\theta_{\varphi(n)}}(\theta). 
$$ 
Using the continuity of $\mcq$ under \ref{hyp:suff:continuity}, this yields $\thvv \in \Mtxt{\thv}$. We deduce under \ref{hyp:suff:lyapunov} that $\thv = \thvv$, and hence $\M{\thv}=\{\thv\}$ under \ref{hyp:suff:singleton}.
\end{proof}

\subsection{Comments on \ref{hyp:main:thv}} \label{sec:app:thm:main}

\begin{proof}[Proof for \Cref{ex:main:pop_em} (Population EM)]
By  \eqref{eq:general_framework:pop_em}, for all $\theta, \theta' \in \Theta$,
$$
\Q[\pop]{\theta}{\theta'} = - \int_{\Yset} \int_{\Xset} p_{\theta}(x | y) \log p_{\theta'}(x, y) p_{\thv}(y) \mu(\rmd x) \mu(\rmd y),
$$
which yields,  in a neighborhood of $(\thv, \thv)$,
$$
\dtwice{2}{2}{\Q[\pop]{\theta}{\theta'}} = - \int_{\Yset} \int_{\Xset} p_{\theta}(x | y) p_{\thv}(y) \partial^2 \log p_{\theta'}(x, y) \mu(\rmd x) \mu(\rmd y).
$$
On the other hand, using that for all $\theta, \theta' \in \Theta$,
$$
        \Q[\pop]{\theta}{\theta'} = - \int_{\Yset} \int_{\Xset} p_{\theta}(x | y) \log p_{\theta'}(x| y) p_{\thv}(y) \mu(\rmd x) \mu(\rmd y) - \int_{\Yset} \log p_{\theta'}(y) p_{\thv}(y) \mu(\rmd y),
$$
    we deduce that in a neighborhood of $(\thv, \thv)$,
    \begin{align*}
        \dtwice{1}{2}{\Q[\pop]{\theta}{\theta'}} &= - \int_{\Yset} \int_{\Xset} p_{\thv}(y) \partial \log p_{\theta'}(x| y) \lrb{\partial p_{\theta}(x | y)}^\top \mu(\rmd x) \mu(\rmd y) \\
        &= - \int_{\Yset} \int_{\Xset} p_{\thv}(y) p_{\theta}(x | y) \partial \log p_{\theta'}(x| y) \lrb{\partial \log p_{\theta}(x | y)}^\top \mu(\rmd x) \mu(\rmd y).
    \end{align*}
    Using the chain rule for Fisher information matrices \cite{zamir,zegers},
    \begin{align}
        \dtwice{2}{2}{\Q[\pop]{\thv}{\thv}} &= -\int_{\Yset} \int_{\Xset} p_{\thv}(x,y) \partial^2 \log p_{\thv}(x, y)  \mu(\rmd x) \mu(\rmd y) = I_{X,Y}(\thv), \label{eq:ex:main:dtwice22:pop}  \\
        \dtwice{1}{2}{\Q[\pop]{\thv}{\thv}} &= -I_{X|Y}(\thv) = -\lr{I_{X,Y}(\thv) - I_Y(\thv)}. \label{eq:ex:main:dtwice12:pop}
    \end{align}
\end{proof}

\begin{proof}[Proofs for \Cref{ex:main:sample_em} (Sample EM)]
    Similarly to the proof of \Cref{ex:main:pop_em}, we deduce from \eqref{eq:general_framework:sample_em} that for all $\theta, \theta'$ in a neighborhood of $\thv$,
    \begin{align*}
        \dtwice{2}{2}{\Q[\samp]{\theta}{\theta'}} &= - \frac 1 k \sum_{i=1}^k \int_{\Xset} p_{\theta}(x | Y_i) \partial^2 \log p_{\theta'}(x, Y_i) \mu(\rmd x), \\
        \dtwice{1}{2}{\Q[\samp]{\theta}{\theta'}} &= - \frac 1 k \sum_{i=1}^k \int_{\Xset} p_{\theta}(x | Y_i) \partial \log p_{\theta'}(x| Y_i) \lrb{\partial \log p_{\theta}(x | Y_i)}^\top \mu(\rmd x),
    \end{align*}
    and therefore,
    \begin{align}
        \dtwice{2}{2}{\Q[\samp]{\thv}{\thv}} &= - \frac 1 k \sum_{i=1}^k \int_{\Xset} p_{\thv}(x | Y_i) \partial^2 \log p_{\thv}(x, Y_i) \mu(\rmd x), \label{eq:ex:main:dtwice22:sample} \\
        \dtwice{1}{2}{\Q[\samp]{\thv}{\thv}} &= \frac 1 k \sum_{i=1}^k \int_{\Xset} p_{\thv}(x | Y_i) \partial^2 \log p_{\thv}(x| Y_i) \mu(\rmd x). \label{eq:ex:main:dtwice12:sample} 
    \end{align}
\begin{lemma}
\label{lem:rhopop_conv}
Assume that $\thv$ is the true parameter of the model, that for all $x,y \in \Xset, \Yset$, the functions $\theta \mapsto p_{\theta}(x|y)$ and $\theta \mapsto p_{\theta}(y)$ are twice differentiable in a neighborhood of $\thv$, and that conditions similar to \cite[Assumption AD.1, p.492]{randal_book} hold to differentiate under the integral sign.  Then, if the corresponding population EM meets \ref{hyp:main:thv}, almost surely the sample EM meets \ref{hyp:main:thv} for sufficiently large $k$ and
\begin{equation*}
        \rhosup^\samp(\Yk) \overset{a.s.}{\longrightarrow} \rhosup^\pop.
    \end{equation*}
Furthermore, if $\partial^2 \log p_\thv(X_1, Y_1), \partial^2 \log p_{\thv}(Y_1) \in \rmL^2(\rset^{q\times q})$, then for all $\delta\in(0;1)$ there exists $C_\delta>0$ such that
    \begin{equation*}
        \underset{k\rightarrow\infty}{\lim \inf} \eqsp \PP\lr{\lrav{\rhosup^\samp(Y_{1:k}) - \rhosup^\pop} \leq \frac{C_\delta} {\sqrt{k}}} \geq 1-\delta.
    \end{equation*}
\end{lemma}
\begin{proof}
    As $\rhosup^\samp\lr{\Yk}$ is not necessarily well-defined in \eqref{eq:main:rhosup}, using \Cref{lemma:tech:rho_a_b} we consider the following definition:
    \begin{equation} \label{eq:ex:main:rho_s}
        \rhosup\lr{\Yk} \eqdef \indic{\tAv\lr{\Yk} \succ 0} \normmat{\tAv\lr{\Yk}^{-1/2} \tBv \lr{\Yk} \tAv\lr{\Yk}^{-1/2}}_2.
    \end{equation}
    Note first that $\rhosup$ is a  measurable function of $\Yk$. Besides, we deduce from the assumptions that the following random variables are integrable:
    \begin{equation*}
        Z_1 \eqdef \int_{\Xset} p_{\thv}(x | Y_1) \partial^2 \log p_{\thv}(x|Y_1) \mu(\rmd x), \quad W_1 \eqdef \int_{\Xset} p_{\thv}(x | Y_1) \partial^2 \log p_{\thv}(x, Y_1) \mu(\rmd x).
    \end{equation*}
    Together with (\ref{eq:ex:main:dtwice22:pop}-\ref{eq:ex:main:dtwice12:pop}) and (\ref{eq:ex:main:dtwice22:sample}-\ref{eq:ex:main:dtwice12:sample}), the strong law of large numbers provides
    \begin{equation} \label{eq:ex:main:as_q}
        \lr{\dtwice{2}{2}{\Qtxt[\samp]{\thv}{\thv}}, \dtwice{1}{2}{\Qtxt[\samp]{\thv}{\thv}}} \overset{a.s.}{\longrightarrow} \lr{\dtwice{2}{2}{\Qtxt[\pop]{\thv}{\thv}}, \dtwice{1}{2}{\Qtxt[\pop]{\thv}{\thv}}}.
    \end{equation}
    By the continuity of the functions used in \eqref{eq:ex:main:rho_s}, this yields that if the corresponding population EM meets \ref{hyp:main:thv}, then, almost surely, the sample EM meets \ref{hyp:main:thv} for sufficiently large $k$, and
    \begin{equation*}
        \rhosup^\samp(\Yk) \overset{a.s.}{\longrightarrow} \rhosup^\pop.
    \end{equation*}

Assume now that $\partial^2 \log p_\thv(X_1, Y_1), \partial^2 \log p_{\thv}(Y_1) \in \rmL^2(\rset^{q\times q})$.
    First, using Jensen's inequality with $\norm{\cdot}_2^2$ provides $W_1\in\rmL^{2}(\rset^{q\times q})$, and thus $Z_1 = W_1 - \partial^2 \log p_\thv(Y_1) \in \rmL^{2}(\rset^{q\times q})$.
    Let $\delta \in (0;1)$ and write $\bar Z_{k} = \sum_{i=1}^k Z_i/k$. Set $\delta'\in(0;1)$ such that $4q^2\delta' \leq \delta$ and write $x(\delta')$ the quantile of order $1-\delta'$ of the standard Gaussian  distribution. Applying the central limit theorem to each component of $Z_1$ provides for all $i,j\in\integerval{q}$ and $k\in\nset^*$,
    \begin{equation*}
        \PP\lr{ \lrav{\bar Z_k(i,j)-\mu_{ij})} \geq \frac {x(\delta') \sigma}{\sqrt{k}} } \leq \PP\lr{ \lrav{\bar Z_k(i,j)-\mu_{ij}} \geq \frac {x(\delta')\sigma_{ij}}{\sqrt{k}} } \underset{k\to\infty}{\longrightarrow} 2\delta',
    \end{equation*}
    where $\mu_{ij}\eqdef \PE[Z_1(i,j)]$, $\sigma_{ij}^2\eqdef \mathrm{Var}\lrb{Z_1(i,j)}$ and $\sigma^2 \eqdef \PE[\normmat{Z_1}_F^2] \vee \PE[\normmat{Z_2}_F^2]$. This yields
    \begin{equation*}
        \underset{k\to\infty}{\lim \sup} \eqsp \PP\lr{ \normmat{\dtwice{1}{2}{\Qtxt[\samp]{\thv}{\thv}} - \dtwice{1}{2}{\Qtxt[\pop]{\thv}{\thv}}}_F \geq \frac {qx(\delta')\sigma}{\sqrt{k}} } \leq 2q^2 \delta'.
    \end{equation*}
    Similarly, the same inequality holds for $\dtwice{2}{2}{\Qtxt[\samp]{\thv}{\thv}}$.
    Let $C_F>0$ such that $\normmat{\cdot}_F \geq C_F \normmat{\cdot}_2$ on $\rset^{q\times q}$.
    We deduce using \eqref{eq:notation:pscal} and \eqref{eq:notation:norm} that
    \begin{multline} \label{eq:ex:main:d22d12}
        \underset{k\to\infty}{\lim \inf} \eqsp \PP\lr{\normmat{\dtwice{2}{2}{\tQtxt[\samp]{\thv}{\thv}} - \dtwice{2}{2}{\tQtxt[\pop]{\thv}{\thv}}}_2 \wedge \normmat{\dtwice{1}{2}{\tQtxt[\samp]{\thv}{\thv}} - \dtwice{1}{2}{\tQtxt[\pop]{\thv}{\thv}}}_2 < \frac {qx(\delta')\sigma}{C_F\sqrt{k}} } \\
        \geq 1-4q^2 \delta' \geq 1-\delta.
    \end{multline}
    Let $\epsilon>0$ and $C >0$ be constants obtained by applying \Cref{lemma:tech:rho_approx} to $(\dtwice{2}{2}{\tQtxt[\pop]{\thv}{\thv}}, \dtwice{1}{2}{\tQtxt[\pop]{\thv}{\thv}})$.
    We deduce from \eqref{eq:ex:main:d22d12} that
    \begin{align*}
        \underset{k\to\infty}{\lim \inf} \eqsp \PP\lr{|\rhosup^\samp\lr{\Yk}-\rhosup^\pop| < \frac {2Cqx(\delta')\sigma}{C_F} \frac 1 {\sqrt{k}} } &\geq 1-\delta,
    \end{align*}
    which proves \eqref{eq:main:comm:clt_rho}.
\end{proof}
%
\end{proof}

\begin{proof}[Proof for \Cref{ex:main:mirror_descent} in page \pageref{ex:main:mirror_descent} (Mirror descent, cont.)]
    
    For all $\theta, \theta' \in \Theta$ in a neighborhood of $\thv$,
    \begin{align}
        \Q{\theta}{\theta'} &= \eta \partial f(\theta)^\top \theta' + \Phi(\theta') - \Phi(\theta) - \partial \Phi(\theta)^\top (\theta' - \theta), \nonumber \\
        \donce{\Q{\theta}{\theta'}} &= \eta \partial f(\theta) + \partial \Phi(\theta') - \partial \Phi(\theta), \nonumber \\
        \dtwice{2}{2}{\Q{\theta}{\theta'}} &= \partial^2 \Phi(\theta'), \label{eq:ex:main:md:q22} \\
        \dtwice{1}{2}{\Q{\theta}{\theta'}} &= \eta \partial^2 f(\theta) - \partial^2 \Phi(\theta). \label{eq:ex:main:md:q12}
    \end{align}
\end{proof}

\begin{proof}[Proof for \Cref{ex:alpha_em} (The $\alpha$-EM algorithm)]

Let $\alpha \in \rset\setminus\{0,1\}$. The function $f_\alpha$ is defined on $\rset_+^*$ by
$ f_\alpha \colon x \mapsto  (1-x^\alpha)/(\alpha(\alpha - 1))$. 
Note that $f_\alpha(1) = 0$ and that for all differentiable functions $g$ taking values in $\rset_+^*$,
\begin{align*}
    \partial \lr{f_\alpha \circ g} &= \frac 1 {1-\alpha} \lrb{1 + \alpha(1-\alpha) f_\alpha \circ g} \partial \lr{\log \circ g}.
\end{align*}
The function $\mcq^{\alpha}$ defined in \eqref{eq:alpha_em:q_alpha} can be written as $\Q[\alpha]{\theta}{\theta'} = - \int_{\Xset} p_{\theta}(x | Y) F^{\alpha}_{\theta}(\theta') \mu(\rmd x)$, 
where $F^{\alpha} \colon (\theta, \theta') \mapsto f_{\alpha} \lr{p_{\theta'}(x,Y) / p_{\theta}(x,Y)}$.
For all $\theta\in\Theta$, $F^{\alpha}_{\theta}(\theta)=0$, and under the assumptions of \Cref{ex:main:pop_em} with $f_\alpha$ instead of $f_0$, for all $\theta, \theta'\in\Theta$,
\begin{align*}
    \partial_2 F^{\alpha}_{\theta}(\theta') &= \frac 1 {1-\alpha} \lrb{1+\alpha(1-\alpha)F^{\alpha}_{\theta}(\theta')} \partial_{\theta'}\log p_{\theta'}(x,Y), \\
    \partial_1 F^{\alpha}_{\theta}(\theta') &= -\frac 1 {1-\alpha} \lrb{1+\alpha(1-\alpha)F^{\alpha}_{\theta}(\theta')} \partial_{\theta}\log p_{\theta}(x,Y).
\end{align*}
We deduce
\begin{alignat*}{2}
    \dtwice{2}{2}{F^{\alpha}_{\theta}(\theta')} &= \frac 1 {1-\alpha} \lrb{1+\alpha(\alpha-1)F^{\alpha}_{\theta}(\theta')} \big( &&\partial_{\theta' \theta'} \log p_{\theta'}(x,Y) \\
    &\eqsp  &&+ \alpha \partial_{\theta'}\log p_{\theta'}(x,Y) \lrb{\partial_{\theta'}\log p_{\theta'}(x,Y)}^\top \big), \\
    \dtwice{1}{2}{F^{\alpha}_{\theta}(\theta')} &= - \frac {\alpha} {1-\alpha} \lrb{1+\alpha(\alpha-1)F^{\alpha}_{\theta}(\theta')} && \partial_{\theta'}\log p_{\theta'}(x,Y) \lrb{\partial_{\theta}\log p_{\theta}(x,Y)}^\top.
\end{alignat*}
At a population level this yields $\dtwice{2}{2}{\Q[\alpha]{\thv}{\thv}} = I_{X, Y}(\thv)$ and 
$$
\dtwice{1}{2}{\Q[\alpha]{\thv}{\thv}} = \frac {\alpha} {1-\alpha} I_{X,Y}(\thv) - \frac 1 {1-\alpha} I_{X|Y}(\thv) = -I_{X,Y}(\thv) +\frac 1 {1-\alpha} I_Y(\thv).
$$


Regarding the value of $\rhosup^\alpha$, note that $\tAv^{-1}\tBv$ is equivalent to the symmetric matrix $\tAv^{-1/2}\tBv\tAv^{-1/2}$ and is therefore diagonalizable.
Besides, we deduce from \eqref{eq:main:comm:a_b_pop_em} and \eqref{eq:alpha_em:a_b_star} that
\begin{equation*}
    \tAv^\alpha = \tAv \quad \mbox{and} \quad \tBv^\alpha = \frac 1 {1-\alpha} \tBv - \frac {\alpha}{1-\alpha} \tAv.
\end{equation*}
This yields $\Spec{(\tAv^\alpha)^{-1}\tBv^\alpha} = g_\alpha (\Spec{\tAv^{-1}\tBv})$ where $g_\alpha(x)\eqdef (x-\alpha)/(1-\alpha)$. We obtain the optimal $\alpha$ by equating $g_\alpha(\rhosup)=-g_\alpha(\rhoinf)$.
\end{proof}

\begin{proof}[Proof for \Cref{ex:main:mirror_prox} in page \pageref{ex:main:mirror_prox} (Mirror prox, cont.)]
    To begin with, the assumptions imply that the corresponding mirror descent meets \ref{hyp:main:regularity}-\ref{hyp:main:thv} and \ref{hyp:suff:compacity}-\ref{hyp:suff:continuity}, and that $\thv=\M{\thv}\in\ri{\Theta}$.
    Together with the fact that the mapping is point-to-point on $\Theta$, (see \Cref{ex:suff:md} in page \pageref{ex:suff:md}), this allows to apply \Cref{lemma:tech:diff_mapping} to mirror descent.
    We deduce the $C^1$-differentiability of $\mcm$ in a neighborhood of $\thv$, where we can write
    \begin{align*}
        \Q[m]{\theta}{\theta'} &= \eta \partial f\lr{\M{\theta}}^\top \theta' + \Phi(\theta') - \Phi(\theta) - \partial \Phi(\theta)^\top (\theta'-\theta), \\
        \donce{\Q[m]{\theta}{\theta'}} &= \eta \partial f\lr{\M{\theta}} + \partial \Phi(\theta') - \partial \Phi(\theta), \\
        \dtwice{2}{2}{\Q[m]{\theta}{\theta'}} &= \partial^2 \Phi(\theta'), \\
        \dtwice{1}{2}{\Q[m]{\theta}{\theta'}} &= -\eta\partial^2 f(\M{\theta}) P \lrb{\partial^2 \tilde{\Phi}(\M{\theta})}^{-1} P^\top \lrb{\eta \partial^2 f(\theta) - \partial^2 \Phi(\theta)} - \partial^2 \Phi(\theta).
    \end{align*}
    This yields $\Av^m = \partial^2 \Phi(\thv)$ and $\Bv^m = \partial^2 \Phi(\thv) - \eta \partial^2 f(\thv) P\tAv^{-1} P^\top \Bv$. Regarding the value of $\rhosup^m$, note that $\tAv^{-1}\tBv$ is equivalent to the symmetric matrix $\tAv^{-1/2}\tBv\tAv^{-1/2}$ and is therefore diagonalizable.
    Together with \eqref{eq:main:comm:mp} this yields
    \begin{equation*}
        \Spec{(\tAv^m)^{-1}\tBv^m} = f\lr{\Spec{\tAv^{-1}\tBv}},
    \end{equation*}
    where $f$ is defined by $f(x)\eqdef x^2-x+1$. Note that $|f(x)|<1$ if and only if $x\in(0;1)$.
    Besides, $\Spec{\tAv^{-1}\tBv}\subset\rset_+^*$ if and only if $u^\top \tAv^{-1/2}\tBv\tAv^{-1/2} u > 0$ for all $u\in\rset^d$, which is equivalent to $u^\top \tBv u > 0$ for all $u\in\rset^d$, by the symmetry of $\tAv^{-1/2}$.
    Finally, $\min_\rset f = 3/4$, which is attained at $x=1/2$, and for all $x\in(0;1)$, $f(x)>x$.
\end{proof}

\section{Technical results} \label{sec:app:tech}

All lemmas below are proved in the Supplementary material (\ref{sec:app:proofs_tech:matrices}).

\subsection{Linear algebra}

\begin{lemma} \label{lemma:tech:domination}
    Let $d\in\nset^*$ and $A,B \in \rmat{d}{d}$ such that $A$ is symmetric positive-definite.
    Then, for all $x, y \in \rset^d$, $x^\top A x \leq x^\top B y \implies \norm{x}_A \leq \rho \norm{y}_A$,  where $\rho \eqdef \normmattxt{A^{-1/2} B A^{-1/2}}_2$.
\end{lemma}

\begin{lemma} \label{lemma:tech:rho_approx}
    Let $d\in\nset^*$ and $A,B \in \rmat{d}{d}$ such that $A$ is symmetric positive-definite.
    Then, there exist $\epsilon>0$ and $C>0$ such that for all symmetric matrices $M \in \rmat{d}{d}$ and for all matrices $N\in\rmat{d}{d}$ verifying $\normmattxt{M}_2 \vee \normmattxt{N}_2\leq\epsilon$,
    \begin{equation*}
        \lrav{\normmattxt{A^{-1/2} B A^{-1/2}}_2 - \normmattxt{(A+M)^{-1/2} (B+N) (A+M)^{-1/2}}_2} \leq C \normmat{M}_2 + C \normmat{N}_2. 
    \end{equation*}
\end{lemma}

\begin{lemma} \label{lemma:tech:rho_a_b}
    Let $d\in\nset^*$, $A \in \rmat{d}{d}$ be a symmetric positive-definite matrix, and $B \in \rmat{d}{d}$ be a symmetric matrix. Then,
    \begin{equation*}
        \rhom{A^{-1}B} = \rhom{BA^{-1}} = \normmat{A^{-1/2} B A^{-1/2}}_2 = \underset{v \in \rset^d}{\sup} \frac{|v^\top B v|} {v^\top A v}.
    \end{equation*}
\end{lemma}

\begin{lemma} \label{lemma:tech:norm_approx}
    Let $d \in \nset^*$ and $S_\star \in \rmat{d}{d}$ be a symmetric positive-definite matrix.
    Then, for all $\epsilon > 0$, there exists $\delta > 0$ such that for all symmetric matrices $S\in\rmat{d}{d}$ verifying $\normmattxt{S - S_\star}_2 < \delta$, and all $x\in\mathbb{R}^d$,
    \begin{equation*}
        (1-\epsilon) \norm{x}_S \leq \norm{x}_{S_\star} \leq (1+\epsilon) \norm{x}_S.
    \end{equation*}
\end{lemma}

\subsection{Minimization}

\begin{lemma} \label{lemma:tech:min_f_compact}
    Let $(\Kset, \rmd)$ be a compact metric space and $f\colon\Kset\rightarrow\rset$ be a continuous function. Write $m\eqdef \min_K f$ and $\Mset\eqdef\argmin_K f$.
    Then, for all $\delta > 0$ there exists $\epsilon > 0$ such that for all $x\in\Kset$,
    \begin{equation*}
        f(x) - m < \epsilon \implies \rmd(x, \Mset) < \delta. 
    \end{equation*}
\end{lemma}

\begin{lemma} \label{lemma:tech:continuity_mapping_set}
    Let $(\Kset, \rmd)$ be a compact metric space and $\mcq\colon\Kset\times\Kset\rightarrow\rset$ be a continuous function. Define for all $x\in\Kset$, $\Mtxt{x} \eqdef \argmin_{x'\in\Kset} \Q{x}{x'}$.
    Then, for all $\xv\in\Kset$ and $\delta>0$, there exists $\delta'>0$ such that for all $x\in\Kset$,
    \begin{equation*}
        \rmd(x,\xv) < \delta' \implies \underset{x'\in\M{x}}{\sup} \rmd(x', \M{\xv}) \leq \delta.
    \end{equation*}
\end{lemma}

\begin{lemma}[Continuity of the minimization mapping] \label{lemma:tech:continuity_mapping}
    Let $(\Kset, \rmd)$ be a compact metric space and $\mcq\colon\Kset\times\Kset\rightarrow\rset$ be a continuous function such that for all $x\in\Kset$, $\Mtxt{x} \eqdef \argmin_{x'\in\Kset} \Q{x}{x'}$ is a singleton.
    Then, the function $\mcm$ is continuous on $\Kset$.
\end{lemma}

\begin{lemma}[Differentiability of the minimization mapping] \label{lemma:tech:diff_mapping}
    Let $d \in \nset^*$, $\Kset$ be a compact set of $\rset^d$, and $\mcq\colon\rset^d\times\rset^d\rightarrow\rset$ be a function continuous on $\Kset \times \Kset$ such that for all $x\in\Kset$, $\Mtxt{x} \eqdef \argmin_{x'\in\Kset} \Q{x}{x'}$ is a singleton.
    Let $x \in \Kset$ such that:
    \begin{enumerate} [(i)]
        \item $x, \M{x} \in \ri{\Kset}$, \label{hyp:tech:dm}
        \item $\partial_2 \mcq$ is well-defined and $C^k$-differentiable in a neighborhood of $(x, \M{x})$ for $k\in\nset^*$, \label{hyp:tech:dm_2}
        \item $\dtwice{2}{2}{\tQtxt{x}{\M{x}}}$ is invertible. \label{hyp:tech:dm_3}
    \end{enumerate}
    Then, the function $\mcm$ is $C^k$-differentiable in a neighborhood of $x$ (considering that the domain lies in the ambient space $\Aff{\Kset}$).
\end{lemma}

\begin{lemma} \label{lemma:tech:sup_f_sup_m}
    Let $(\Kset, \rmd)$ be a compact metric space and $f\colon\Kset\rightarrow\rset$ be a continuous function.
    Then, for all $\delta>0$, there exists $\epsilon>0$ such that for all functions $\hat{f}\colon\Kset\rightarrow\rset$,
    \begin{equation*}
        \underset{\Kset}{\sup} \lrav{\hat{f} - f} < \epsilon \implies \underset{y\in\hat{\Mset}}{\sup} \eqsp \rmd(y,\Mset) \leq \delta,
    \end{equation*}
    with $\hat{\Mset}\eqdef \argmin_{\Kset} \hat{f}$, and the convention that the supremum over an empty set is equal to minus infinity.
\end{lemma}

\begin{lemma} \label{lemma:tech:sup_q_sup_m}
    Let $(\Kset, \rmd)$ be a compact metric space and $\mcq\colon\Kset\times\Kset\rightarrow\rset$ be a continuous function such that for all $x\in\Kset$, $\Mtxt{x} \eqdef \argmin_{x'\in\Kset} \Q{x}{x'}$ is a singleton.
    Then, for all $\delta > 0$, there exists $\epsilon > 0$ such that for all functions $\hat{\mcq}\colon \Kset\times\Kset \rightarrow \rset$,
    \begin{equation*}
        \underset{\Kset\times\Kset}{\sup} \lrav{\hat{\mcq} - \mcq} < \epsilon \implies \underset{x\in\Kset \;;\;y\in\hat{\mcm}(x) }{\sup} \rmd(y,\M{x}) \leq \delta,
    \end{equation*}
    where $\hat{\mcm}$ is defined by $\hat{\mcm}(x)\eqdef \argmin_{x'\in\Kset} \hat{\mcq}_{x}(x')$, and with the convention that the supremum over an empty set is equal to minus infinity.
\end{lemma}

\subsection{Convexity}

\begin{lemma} \label{lemma:tech:biconjugate}
    Let $d\in\nset^*$, $\Kset$ be a bounded convex set of $\rset^d$, and $f\colon\Kset\rightarrow\rset$ be a continuous function.
    Assume that $f$ has a unique minimizer $\xv$ on $\Kset$, that $\xv\in\ri{K}$, that $f$ is $C^2$-differentiable in a neighborhood of $\xv$ and that $\partial^2 f(\xv) \succ 0$.
    
    Then, $\xv$ is the unique minimizer of $f^{**}$ on $\Kset$ and $f^{**}$ is equal to $f$ in a neighborhood of $\xv$, where $f^{**}$ denotes the biconjugate of $f$ (see \cite[Section 11, p.473]{rockafellar_variational}).
\end{lemma}

\bibliographystyle{apalike} 
\bibliography{bibliography}

\section{Supplementary material}
\label{sec:supplementary}

\subsection{Linear algebra} \label{sec:app:proofs_tech:matrices}

\begin{proof}[Proof of \Cref{lemma:tech:domination}]
    Let $x,y\in\rset^d$ such that $0<x^\top A x \leq x^\top B y$. The Cauchy-Schwarz inequality provides
    \begin{equation*}
        \norm{x}_A^2=x^\top A x \leq x^\top B y = (A^{1/2} x)^\top A^{-1/2} B y \leq \normtxt{A^{1/2} x}_2 \normtxt{A^{-1/2} B y}_2.
    \end{equation*}
    Using that $\normtxt{A^{-1/2} B y}_2 \leq \normmattxt{A^{-1/2} B A^{-1/2}}_2 \normtxt{A^{1/2} y}_2$, we deduce
    \begin{equation*}
        \norm{x}_{A}^2 \leq \rho \normtxt{A^{1/2} x}_2 \normtxt{A^{1/2} y}_2
        = \rho \norm{x}_A \norm{y}_A,
    \end{equation*}
where $\rho = \normmattxt{A^{-1/2} B A^{-1/2}}_2$.
\end{proof}

\begin{lemma} \label{lemma:tech:rho_approx:square}
    Let $d\in\nset^*$ and $A \in \rmat{d}{d}$ be a symmetric positive-definite matrix.
    Then, there exist $\epsilon>0$ and $C>0$ such that for all
    symmetric matrices $M\in\rmat{d}{d}$ verifying $\normmattxt{M}_2
    \leq \epsilon$, the matrix $A+M$ is symmetric positive-definite and 
    \begin{equation*}
        \normmat{A^{1/2} - \lr{A+M}^{1/2}}_2 \leq C \normmat{M}_2. 
    \end{equation*}
\end{lemma}
\begin{proof}
    Write $\lambda \eqdef \min \Spec{A} > 0$. Let $M\in\rmat{d}{d}$ be a symmetric matrix such that $\normmattxt{M}_2\leq\lambda/2$. The matrix $A+M$ is then symmetric positive-definite and we can define its square root.
    By the symmetry of $A_M \eqdef (A+M)^{1/2} - A^{1/2}$, there exist $\mu\in\Spec{A_M}$ such that $|\mu| = \normmattxt{A_M}_2$ and $x\in\rmat{d}{d}$ such that $A_M x=\mu x$ and $\norm{x}_2=1$. This yields
    \begin{align*}
        x^\top M x &= x^\top (A+M)^{1/2}\lr{(A+M)^{1/2} - A^{1/2}} x + x^\top \lr{(A+M)^{1/2} - A^{1/2}}A^{1/2} x \\
        &= \mu x^\top \lr{(A+M)^{1/2} + A^{1/2}} x.
    \end{align*}
    We deduce, using that $x^T (A+M)^{1/2} x\geq 0$ and $\lambda^{1/2} = \min \Spec{A^{1/2}}$, 
    \begin{equation*}
        \normmat{M}_2 \geq |x^\top M x| = |\mu| x^\top \lr{(A+M)^{1/2} + A^{1/2}} x \geq \normmat{A_M}_2 x^\top  A^{1/2} x  \geq \normmat{A_M}_2 \lambda^{1/2},
    \end{equation*}
    and hence the result with $\epsilon=\lambda/2$ and $C=\lambda^{-1/2}$.
\end{proof}

\begin{lemma} \label{lemma:tech:rho_approx:inv}
    Let $d\in\nset^*$ and $A \in \rmat{d}{d}$ be an invertible matrix.
    Then, there exist $\epsilon>0$ and $C>0$ such that for all matrices $M\in\rmat{d}{d}$ verifying $\normmattxt{M}_2 \leq \epsilon$,
    \begin{equation*}
        \normmat{A^{-1} - \lr{A+M}^{-1}}_2 \leq C \normmat{M}_2. 
    \end{equation*}
\end{lemma}
\begin{proof}
    Let $M\in\rmat{d}{d}$ such that $\normmattxt{M}_2\leq\normmattxt{A^{-1}}_2^{-1}/2$. We can then write
    \begin{equation*}
        A^{-1} - \lr{A+M}^{-1} = A^{-1} - A^{-1}(I + MA^{-1})^{-1}
        = -A^{-1} MA^{-1} \sum_{k=0}^{\infty} (-MA^{-1})^k.
    \end{equation*}
    This yields the result with $\epsilon = \normmattxt{A^{-1}}_2^{-1}/2$ and $C=2\normmattxt{A^{-1}}_2^2$.
\end{proof}

\begin{proof}[Proof of \Cref{lemma:tech:rho_approx}]
    Applying \Cref{lemma:tech:rho_approx:square} to $A^{-1}$ and
    \Cref{lemma:tech:rho_approx:inv} to $A$ provides the existence of
    $\epsilon>0$ and $C>1$ such that for all symmetric matrices
    $S\in\rmat{d}{d}$ verifying $\normmattxt{S}_2\leq\epsilon$, the
    matrix $A^{-1}+S$ is symmetric positive-definite and 
    \begin{align}
        \normmat{\lr{A^{-1}}^{1/2} - \lr{A^{-1}+S}^{1/2}}_2 &\leq C
                                                              \normmat{S}_2, \label{eq:tech:square}\\
              \normmat{A^{-1} - \lr{A+S}^{-1}}_2 &\leq C \normmat{S}_2. \label{eq:tech:inv} 
    \end{align}
    Let $M\in\rmat{d}{d}$ be a symmetric matrix such that $\normmattxt{M}_2\leq \epsilon/C \leq \epsilon$.
    By \eqref{eq:tech:inv} we can then define $M' \eqdef A^{-1} -
    (A+M)^{-1}$, which verifies $\normmattxt{M'}_2\leq C
    \normmattxt{M}_2 \leq \epsilon$. Choosing $S=-M'$, we deduce that  $A^{-1} - M'=(A+M)^{-1}$ 
    is symmetric positive-definite  and by \eqref{eq:tech:square},
    \begin{align*}
        \normmat{\lr{A}^{-1/2} - \lr{A+M}^{-1/2}}_2 &= \normmat{\lr{A^{-1}}^{1/2} - \lr{(A+M)^{-1}}^{1/2}}_2 \\
        &= \normmat{\lr{A^{-1}}^{1/2} - \lr{A^{-1}-M'}^{1/2}}_2 \\
        &\leq C \normmat{M'}_2 \leq C^2 \normmat{M}_2.
    \end{align*}
    This concludes the proof up to simple algebra, noting that for all matrices $N\in\rmat{d}{d}$,
    \begin{align*}
        A^{-1/2} B A^{-1/2} - &(A+M)^{-1/2} (B+N) (A+M)^{-1/2} \\
        = &\lr{A^{-1/2} - (A+M)^{-1/2}}B A^{-1/2} - (A+M)^{-1/2}N A^{-1/2} \\
        &+ (A+M)^{-1/2}(B+N) \lr{A^{-1/2} -  (A+M)^{-1/2}}.
    \end{align*}
\end{proof}

\begin{proof}[Proof of \Cref{lemma:tech:rho_a_b}]
    The symmetry of $A^{-1}$ and $B$ provides the first equality.
    As $A^{-1} B = A^{-1/2} (A^{-1/2} B A^{-1/2}) A^{1/2}$,  $A^{-1} B$ and $A^{-1/2} B A^{-1/2}$ are similar and then
    \begin{equation*}
        \rhomtxt{A^{-1}B} = \rhomtxt{A^{-1/2} B A^{-1/2}}.
    \end{equation*}
    Since $A^{-1/2} B A^{-1/2}$ is symmetric, we can write
    \begin{align*}
        \rhom{A^{-1/2} B A^{-1/2}} &= \max\lrav{\Spec{A^{-1/2} B A^{-1/2}}} \\
        &= \underset{x\in\rset^d}{\sup} \frac {\lrav{x^\top A^{-1/2} B A^{-1/2} x}} {x^\top x}
        = \normmat{A^{-1/2} B A^{-1/2}}_2.
    \end{align*}
    This provides the second equality, along with the third one by considering the change of variables $v = A^{-1/2}x$.
\end{proof}

\begin{proof}[Proof of \Cref{lemma:tech:norm_approx}]

    For all $x \in \rset^d$, we have 
    $$
    \norm{x}_2\leq \normmattxt{S_{\star}^{-1/2}}_2 \norm{S_{\star}^{1/2}x}_2=\normmattxt{S_{\star}^{-1/2}}_2 \norm{x}_{S_{\star}}
    $$
This implies 
$$ 
|\norm{x}_S^2 - \norm{x}_{S_\star}^2|  = | \pscal{x}{(S-S_\star) x}| \leq \normmattxt{S-S_\star}_2 \norm{x}_2^2 \leq \norm{x}_{S_\star}^2 \normmattxt{S-S_\star}_2 \normmattxt{S_{\star}^{-1/2}}_2^2,
$$
which yields, for all $S$ such that $\normmattxt{S-S_\star}_2\leq \delta$,
$$
\left(1 + \delta \normmattxt{S_{\star}^{-1/2}}_2^2\right)^{-1/2}\norm{x}_{S}\leq  \norm{x}_{S_\star} \leq \left(1 - \delta \normmattxt{S_{\star}^{-1/2}}_2^2\right)^{-1/2}\norm{x}_{S}.
$$
The proof follows. 
\end{proof}

\subsection{Minimization} \label{sec:app:proofs_tech:minimization}

\begin{proof}[Proof of \Cref{lemma:tech:min_f_compact}]
    Let $\delta>0$. The function $\trmd$ defined on $\Kset$ by $\trmd(x) \eqdef \rmd(x,\Mset)$ being continuous, the set $\tilde{\Kset} \eqdef \trmd^{-1}([\delta, +\infty[)$ is compact, as the intersection of a closed set with a compact set.
    By the continuity of  $f$ we deduce the existence of $x_0 \in \tilde{\Kset}$ such that
    \begin{equation*}
        \epsilon \eqdef \underset{x \in \tilde{\Kset}}{\inf} f - m = f(x_0) - m > 0.
    \end{equation*}
\end{proof}

\begin{proof}[Proof of \Cref{lemma:tech:continuity_mapping_set}]
    Let $\xv\in\Kset$ and $\delta>0$. By the compacity of $\Kset$ and the continuity of $\Q{\xv}{\cdot}$, there exists $\xvv\in\M{\xv}$. Besides, \Cref{lemma:tech:min_f_compact} applied to $\Q{\xv}{\cdot}$ provides the existence of $\epsilon>0$ such that for all $x'\in\Kset$,
    \begin{equation} \label{lemma:tech:conti_ms}
        \Q{\xv}{x'} - \Q{\xv}{\xvv} < \epsilon \implies \rmd(x', \M{\xv}) < \delta.
    \end{equation}
    Moreover, by the uniform continuity of $\mcq$ on $(\Kset\times\Kset, \trmd)$, where $\trmd((y,y'), (z,z'))\eqdef\rmd(y,z) + \rmd(y',z')$, there exists $\delta'>0$ such that for all $y, y', z, z' \in\Kset$,
    \begin{equation*}
        \trmd((y,y'), (z,z')) < \delta' \implies \lrav{\Q{y}{y'} - \Q{z}{z'}} < \epsilon/2.
    \end{equation*}
    We deduce that for all $x\in\Kset$ such that $\rmd(x,\xv)<\delta'$, for all $x'\in\M{x}$,
    \begin{equation*}
        \Q{\xv}{x'} < \Q{x}{x'} + \epsilon/2 \leq \Q{x}{\xvv} + \epsilon/2 < \Q{\xv}{\xvv} + \epsilon,
    \end{equation*}
    and thus $\rmd(x', \M{\xv})<\delta$ by \eqref{lemma:tech:conti_ms}.
\end{proof}

\begin{proof}[Proof of \Cref{lemma:tech:continuity_mapping}] \renewcommand{\qedsymbol}{}
    This is a direct consequence of \Cref{lemma:tech:continuity_mapping_set}.
\end{proof}

\begin{proof}[Proof of \Cref{lemma:tech:diff_mapping}]
    We first prove the lemma for $k=1$.
    Write $\Vset$ the direction of $\Aff{\Kset}$ and for all $v \in \Vset, \epsilon>0$, write $\Ball{v}{\epsilon} \eqdef \settxt{w \in \Vset}{\normtxt{v-w}_2<\epsilon}$. Note that the ball is defined as a subset of $\Vset$.
    
    Under \ref{hyp:tech:dm} and \ref{hyp:tech:dm_2} there exists $\epsilon_0>0$ such that $\partial_2 \mcq$ is $C^1$-differentiable on $\Ball{x}{2\epsilon_0}\times\Ball{\Mtxt{x}}{2\epsilon_0} \subset \ri{\Kset}\times\ri{\Kset}$.
    Moreover, the function $\mcm$ is continuous on $\Kset$ by \Cref{lemma:tech:continuity_mapping}, which yields the existence of $\epsilon_1\in(0;\epsilon_0)$ such that $y\in\Ball{x}{\epsilon_1}$ implies $\Mtxt{y}\in\Ball{\M{x}}{\epsilon_0}$.
    By \ref{hyp:tech:dm_3} there also exists $\epsilon_2\in(0;\epsilon_1)$ such that for all $y \in \Ball{x}{\epsilon_2}$, the matrix $\dtwice{2}{2}{\tQtxt{y}{\M{y}}}$ is invertible.

    Let $y \in \Ball{x}{\epsilon_2}$ and set $\epsilon>0$ such that $\Ball{y}{\epsilon}\subset\Ball{x}{\epsilon_2}$. For all $h\in\Ball{0}{\epsilon}$, the fact that $\M{y}, \M{y+h}\in\ri{\Kset}$ provides (see the proof of \Cref{thm:main:lower_bound}):
    \begin{equation*}
        \tmca(h) \lr{\tM{y+h}-\tM{y}} = \tmcb(h) \tilde{h},
    \end{equation*}
    where the functions $\mca$ and $\mcb$ are defined on $\Ball{0}{\epsilon}$ by
    \begin{align*}
        \mca(h) &\eqdef \int_0^1 \dtwice{1}{1}{\Q{y+h}{s\M{y+h}+(1-s)\M{y}}} \rmd s, \\
        \mcb(h) &\eqdef -\int_0^1 \dtwice{1}{2}{\Q{sy+(1-s)(y+h)}{\M{y}}} \rmd s.
    \end{align*}
    Besides, $\tmca(0) = \dtwice{2}{2}{\tQtxt{y}{\M{y}}}$ is invertible by the definition of $\epsilon_2$, and the functions $\mca$, $\mcb$ are continuous on $\Ball{0}{\epsilon}$ by the continuity of $\mcm$ and the uniform continuity of $\dtwice{1}{2}{\mcq}, \dtwice{2}{2}{\mcq}$ on $\bar{\mathbf{B}}(x, \epsilon_0)\times\bar{\mathbf{B}}(\Mtxt{x}, \epsilon_0)$.
    Therefore, there exists $\epsilon' \in(0;\epsilon)$ such that on $\Ball{0}{\epsilon'}$,
    \begin{equation*}
        \tM{y+h}-\tM{y} = \tmca(h)^{-1}\tmcb(h) \tilde{h} = \tmca(0)^{-1}\tmcb(0) \tilde{h} + \lito(\norm{h}).
    \end{equation*}
    We deduce
    \begin{equation*}
        \partial\M{y} = -P \lrb{\dtwice{1}{1}{\tQ{y}{\M{y}}}}^{-1} \dtwice{1}{2}{\tQ{y}{\M{y}}} P^\top,
    \end{equation*}
    where $P$ is defined as in \Cref{sec:main}.
    The case $k>1$ follows by induction, using the above expression and the $C^{\infty}$-differentiability of the matrix inverse.
\end{proof}

\begin{proof}[Proof of \Cref{lemma:tech:sup_f_sup_m}]
    Let $\delta>0$ and $\xv\in\Mset$. By \Cref{lemma:tech:min_f_compact} there exists $\epsilon>0$ such that for all $y\in\Kset$,
    \begin{equation} \label{eq:tech:sup_f_sup_m}
        f(y) - m < 2\epsilon \implies \rmd(y,\Mset) < \delta.
    \end{equation}
    Let $\hat{f}$ be a real-valued function defined on $\Kset$ such that $\sup_{\Kset} |\hat{f}-f|<\epsilon$. Then, for all $y\in\hat{\Mset}$,
    \begin{equation*}
        f(y) < \hat{f}(y) + \epsilon \leq \hat{f}(\xv) + \epsilon < f(\xv) + 2\epsilon,
    \end{equation*}
    which implies $\rmd(y,\Mset) < \delta$ by \eqref{eq:tech:sup_f_sup_m}.
\end{proof}

\begin{lemma} \label{lemma:tech:min_q_compact}
    Let $(\Kset, \rmd)$ be a compact metric space and $\mcq\colon\Kset\times\Kset\rightarrow\rset$ be a continuous function such that for all $x\in\Kset$, $\Mtxt{x} \eqdef \argmin_{x'\in\Kset} \Q{x}{x'}$ is a singleton.
    Then, for all $\delta > 0$, there exists $\epsilon > 0$ such that for all $x, x' \in \Kset$,
    \begin{equation*}
        \Q{x}{x'} - \Q{x}{\M{x}} < \epsilon \implies \rmd(x',\M{x}) < \delta.
    \end{equation*}
\end{lemma}
\begin{proof}
    By \Cref{lemma:tech:continuity_mapping}, the function $\trmd$ defined on $\Kset^2$ by $\trmd(x,x') \eqdef \rmd(\Mtxt{x}, x')$ is continuous.
    Let $\delta > 0$. By the compacity of $\tilde{\Kset} \eqdef \trmd^{-1}([\delta, +\infty[)$ and the continuity of $\mcq$, there exists $(x_0, x_0') \in \tilde{\Kset}$ such that
    \begin{equation*}
        \epsilon \eqdef \underset{(x,x') \in \tilde{\Kset}}{\inf} \lrb{\Q{x}{x'} - \Q{x}{\M{x}}} = \Q{x_0}{x_0'} - \Q{x_0}{\M{x_0}} > 0.
    \end{equation*}
\end{proof}

\begin{proof}[Proof of \Cref{lemma:tech:sup_q_sup_m}]
    Let $\delta>0$. By \Cref{lemma:tech:min_q_compact} there exists $\epsilon>0$ such that for all $x, y\in\Kset$,
    \begin{equation} \label{eq:tech:sup_q_sup_m}
        \Q{x}{y} - \Q{x}{\M{x}} < 2\epsilon \implies \rmd(y,\M{x}) < \delta.
    \end{equation}
    Let $\hat{\mcq}$ be a real-valued function defined on $\Kset\times\Kset$ such that $\sup_{\Kset\times\Kset} |\hat{\mcq}-\mcq|<\epsilon$. Then, for all $x,y\in\Kset$ such that $y\in\hat{\mcm}(x)$,
    \begin{equation*}
        \Q{x}{y} < \hat{\mcq}_{x}(y) + \epsilon \leq \hat{\mcq}_{x}\lr{\M{x}} + \epsilon < \Q{x}{\M{x}} + 2\epsilon,
    \end{equation*}
    which implies $\rmd(y,\M{x}) < \delta$ by \eqref{eq:tech:sup_q_sup_m}.
\end{proof}

\subsection{Convexity} \label{sec:app:proofs_tech:convexity}

\begin{proof}[Proof of \Cref{lemma:tech:biconjugate}]
    For all $\epsilon>0$, write $\Ball{\xv}{\epsilon}\eqdef\set{x\in\Kset}{\norm{x-\xv}_2<\epsilon}$.
    To begin with, note that $f(\xv)=f^{**}(\xv)$. Indeed, $f(\xv)\geq f^{**}(\xv)$ by definition of the biconjugate, and the constant $f(\xv)$ is an affine minorant of $f$, which provides $f^{**} \geq f(\xv)$.
    
    We now prove that $\xv$ is the unique minimizer of $f^{**}$ on $\Kset$.
    By assumption there exists $\epsilon_0>0$ such that $f$ is $C^2$-differentiable and $\partial^2 f \succ 0$ on $\Ball{\xv}{\epsilon_0}$, which implies the convexity of $f$ on that neighborhood.
    Besides, by \Cref{lemma:tech:min_f_compact} there exists $\delta>0$ such that for all $x\in\Kset$,
    \begin{equation} \label{eq:tech:min_f}
        f(x) - f(\xv) < \delta \implies \norm{x-\xv}_2 < \epsilon_0.
    \end{equation}
    Let $\epsilon_1 \in(0;\epsilon_0)$ such that $f(\Ball{\xv}{\epsilon_1})\subset\Ball{f(\xv)}{\delta/2}$.
    As $\xv\in\ri{K}$, for all $x\in\Kset$, $\partial f(\xv)^\top (x-\xv)=0$.
    By the boundedness of $\Kset$ and the $C^1$-differentiability of $f$ on $\Ball{\xv}{\epsilon_1}$ this provides the existence of $\epsilon_2\in(0;\epsilon_1)$ such that for all $x\in\Ball{\xv}{\epsilon_2}$ and $y\in\Kset$, $\partial f(x)^\top (y-x) \leq \delta/2$.
    Together with \eqref{eq:tech:min_f} this yields for all $x\in\Ball{\xv}{\epsilon_2}$ and $y\in\Kset\setminus\Ball{\xv}{\epsilon_0}$,
    \begin{equation} \label{eq:tech:lowerbound_f}
        f(y) \geq \delta/2 + \delta/2 + f(\xv) \geq f(x_0) + \partial f(x)^\top (y-x).
    \end{equation}
    By the convexity of $f$ on $\Ball{\xv}{\epsilon_0}$, the same inequality holds for all $y\in\Ball{\xv}{\epsilon_0}$.
    We deduce from \eqref{eq:tech:lowerbound_f} that for all $x\in\Ball{\xv}{\epsilon_2}$, $y\in\Kset$,
    \begin{equation} \label{eq:tech:lowerbound_biconjugate}
        f^{**}(y)\geq f(x) + \partial f(x)^\top (y-x).
    \end{equation}
    Let $y\in\Kset\setminus\{\xv\}$ and set $t\in(0;1)$ such that $x(t)\eqdef \xv + t(y-\xv) \in\Ball{\xv}{\epsilon}$. By the convexity of $f$ on $\Ball{\xv}{\epsilon}$,
    \begin{equation*}
        \partial f(x(t))^\top (y-\xv) = \frac 1 t \partial f(x(t))^\top (x(t)-\xv) \geq \frac 1 t \partial f(\xv)^\top (x(t)-\xv) = 0.
    \end{equation*}
    Using \eqref{eq:tech:lowerbound_biconjugate} with $x = x(t)$ then yields $f^{**}(y)\geq f(x(t))>f(\xv)=f^{**}(\xv)$, which proves that $\xv$ is the only minimizer of $f^{**}$ on $\Kset$.

    Finally, for all $x\in\Ball{\xv}{\epsilon_2}$, using \eqref{eq:tech:lowerbound_biconjugate} with $y=x$ provides $f^{**}(x)\geq f(x)$, and hence $f=f^{**}$ on $\Ball{\xv}{\epsilon_2}$.
\end{proof}

\section*{Acknowledgement}
Many results presented in this paper were obtained during the first year of the Ph.D. of Rayan Charrier, before his resignation for personal reasons.

\end{document}